\documentclass[sigconf]{acmart}

\usepackage{array}
\usepackage{multirow}
\usepackage{algorithm}
\usepackage{algpseudocode}
\usepackage{subfigure}
\newtheorem{assumption}{Assumption}

\AtBeginDocument{%
  }

\setcopyright{acmlicensed}
\copyrightyear{2018}
\acmYear{2018}
\acmDOI{XXXXXXX.XXXXXXX}
\acmConference[Conference acronym 'XX]{Make sure to enter the correct
  conference title from your rights confirmation email}{June 03--05,
  2018}{Woodstock, NY}

\acmISBN{978-1-4503-XXXX-X/2018/06}

\begin{document}

\title{FedWCM: Unleashing the Potential of Momentum-based Federated Learning in Long-Tailed Scenarios}

\author{Tianle Li}
\email{2200271015@email.szu.edu.cn}
\affiliation{%
  \institution{Shenzhen University}
  \city{Shenzhen}
  \country{China}
}
\authornote{These authors contributed equally to this work.}

\author{Yongzhi Huang}
\email{huangyongzhi@email.szu.edu.cn}
\affiliation{%
  \institution{The Hong Kong University of Science and Technology (Guangzhou)}
  \city{Guangzhou}
  \country{China}
}
\authornotemark[1]

\author{Linshan Jiang}
\email{linshan@nus.edu.sg}
\affiliation{%
  \institution{National University of Singapore}
  \city{Singapore}
  \country{Singapore}
}
\authornote{Co-corresponding author. For academic inquiries, please contact this author.}

\author{Qipeng Xie}
\email{qxieaf@connect.ust.hk}
\affiliation{%
  \institution{The Hong Kong University of Science and Technology (Guangzhou)}
  \city{Guangzhou}
  \country{China}
}

\author{Chang Liu}
\email{liuc0063@e.ntu.edu.sg}
\affiliation{%
  \institution{Nanyang Technological University}
  \city{Singapore}
  \country{Singapore}
}

\author{Wenfeng Du}
\email{duwf@szu.edu.cn}
\affiliation{%
  \institution{Shenzhen University}
  \city{Shenzhen}
  \country{China}
}

\author{Lu Wang}
\email{wanglu@szu.edu.cn}
\affiliation{%
  \institution{Shenzhen University}
  \city{Shenzhen}
  \country{China}
}
\authornote{Main corresponding author.}

\author{Kaishun Wu}
\email{wuks@hkust-gz.edu.cn}
\affiliation{%
  \institution{The Hong Kong University of Science and Technology (Guangzhou)}
  \city{Guangzhou}
  \country{China}
}

\begin{abstract}
  Federated Learning (FL) enables decentralized model training while preserving data privacy. Despite its benefits, FL faces challenges with non-identically distributed (non-IID) data, especially in long-tailed scenarios with imbalanced class samples. Momentum-based FL methods, often used to accelerate FL convergence, struggle with these distributions, resulting in biased models and making FL hard to converge. To understand this challenge, we conduct extensive investigations into this phenomenon, accompanied by a layer-wise analysis of neural network behavior. Based on these insights, we propose FedWCM, a method that dynamically adjusts momentum using global and per-round data to correct directional biases introduced by long-tailed distributions. Extensive experiments show that FedWCM resolves non-convergence issues and outperforms existing methods, enhancing FL's efficiency and effectiveness in handling client heterogeneity and data imbalance.
\end{abstract}


\maketitle
\section{Introduction}

Federated Learning (FL) \cite{mcmahan2017communication} enables collaborative model training across multiple parties without centralizing data, thus ensuring privacy by sharing only model updates with a central server, which aggregates and redistributes a global model. A challenge in these environments is the non-independently and identically distributed (non-IID) data across parties \cite{zhu2021federated}, which can slow convergence and reduce performance \cite{karimireddy2020scaffold}. Solutions include data augmentation \cite{jeong2018communication,goetz2020federated}, personalized federated learning \cite{fallah2020personalized}, and clustering \cite{sattler2020clustered}. Momentum-based methods—applied at the server \cite{reddi2020adaptive,sun2024role}, client \cite{karimireddy2020mime,xu2021fedcm}, or both \cite{wang2023accelerating}—help mitigate non-IID issues by speeding convergence through historical gradient accumulation, offering a simple and efficient solution without additional computation.

However, in real-world situations, lots of non-IID data may meet a long-tailed distribution~\cite{shuai2022balancefl}, where the global class distribution is imbalanced: head classes contain abundant samples, whereas tail classes have relatively few, leading to a bias in client models towards the head classes~\cite{shang2022}. For example, in some IoT applications, such as smart homes and healthcare monitoring systems, common activities such as sitting and walking dominate the data. At the same time, critical events like falls or specific medical conditions are rare. As the complexity of addressing long-tailed non-IID data will be significantly increased, it remains an open problem that is far from being resolved~\cite{chen2022towards}.

So far, there are very limited studies focusing on non-IID data collocated with long-tailed data distribution to tackle the dual challenges~\cite{zhang2023survey}. BalanceFL~\cite{shuai2022balancefl} corrects local training through a local update scheme, forcing the local model to behave as if it were trained on a uniform distribution. The Fed-GraB method~\cite{xiao2024fed} addresses global long-tailed distribution by employing a self-adjusting gradient balancer and a prior analyzer. CReFF~\cite{shang2022federated} alleviates the bias by retraining classifiers on federated features. CLIP2FL~\cite{shi2023clip} leverages the strengths of the CLIP model to enhance feature representation. However, on the one hand, these methods primarily mitigate issues arising from long-tail distributions without any special design for the convergence speed; on the other hand, these approaches often require disruptive modifications to the methods themselves, making it challenging to integrate with other potential techniques to accelerate the convergence process.

To tackle these challenges, we aim to unleash the potential of the momentum-based approach to improve the learning performance on long-tailed non-IID data distribution.  We attempt to adopt a naive momentum-based approach~\cite{xu2021fedcm,cheng2023momentum} as a start point, and we find that they perform terribly in various settings, where the convergence details are shown in the motivation. It brings us a question: \textit{can we design novel plug-in modules that can fit momentum-based FL algorithms to deal with the long-tailed non-IID data challenges while keeping their strength?}

We delve into this question by first analyzing the convergence difficulties caused by long-tailed distributions, identifying the core issue as momentum-induced global direction distortion. To tackle this issue, we propose an improved momentum-based federated learning approach called FedWCM that introduces a novel momentum adjustment mechanism. This mechanism consists of two key adaptive strategies: first, we adjust the aggregation method of momentum by utilizing global insights to refine how momentum is collected and integrated across clients; second, we modify the degree to which momentum is applied, dynamically tailoring it based on comprehensive global and local data assessments. By implementing these strategies, we ensure that the momentum mechanism effectively mitigates the negative impacts of long-tailed data distributions. This dual approach preserves momentum’s acceleration benefits while addressing the inherent challenges.

Our main contributions are summarized as follows:
\begin{itemize}
\item To the best of our knowledge, we are the first to identify the convergence challenges faced by momentum-based federated learning under long-tailed non-IID data distributions. Through intuitive analysis, we have demonstrated that this issue arises since the momentum induced causes the global aggregation direction to be skewed by the bias introduced by long-tailed data, thereby hindering model convergence.
\item Inspired by the above insights, we propose FedWCM, in which global information is incorporated to adaptively adjust the previously fixed momentum value and momentum aggregation weights on a per-round basis, resulting in a balanced global momentum and enabling it to exert an appropriate influence in the next round based on its situation. With this novel design, FedWCM can leverage the advantages of momentum while avoiding the non-convergence issues that arise in long-tailed scenarios.

\item We perform the convergence analysis and conduct extensive experiments to demonstrate the effectiveness of FedWCM. Our theoretical proof indicates that FedWCM has the same convergence rate as FedCM~\cite{xu2021fedcm}. Furthermore, our experiments result shows that FedWCM exceeds the performance of state-of-the-art algorithms on various datasets and outperforms some potential long-tail enhancement methods when integrated with FedCM.
\end{itemize}
\section{Related work}
\subsection{Momentum-based Federated Learning}

To address client drift in heterogeneous federated learning, researchers have proposed various momentum-based methods. SCAFFOLD ~\cite{karimireddy2020scaffold} introduces control variates to correct client update biases, Mime~\cite{karimireddy2020mime} combines momentum SGD with variance reduction techniques, FedDyn~\cite{acar2021federated} balances global and local objectives through dynamic regularization, and SlowMo~\cite{wang2019slowmo} incorporates momentum updates on the server side. These methods mitigate client drift and improve model convergence through different momentum strategies. 

\subsection{Federated Long-tailed Learning}
Federated long-tailed learning~\cite{li2024federated} aims to address global class imbalance issues in federated environments. Some studies explore federated long-tailed learning solutions from a meta-learning perspective~\cite{qian2023long}, aiming to improve model adaptability to new tasks. Research has also focused on client selection and aggregation strategy improvements~ \cite{zhang2021dubhe,geng2022bearing} to better handle long-tailed data. Additionally, some researchers~\cite{xiao2024fedloge} have attempted to address the challenge of model personalization, aiming to improve the accuracy of local models in long-tailed environments. While these methods have made progress in mitigating class imbalance issues, they primarily focus on addressing problems arising from long-tailed data, with less consideration given to client heterogeneity.

There are limited studies working on non-IID data collocated with long-tailed data distribution. Model decoupling methods like FedGraB~\cite{xiao2024fed} characterize global long-tailed distributions under privacy constraints and adjust local learning strategies. Feature enhancement methods have been explored, with CReFF~ \cite{shang2022federated} retraining classifiers on federated features, and CLIP2FL~\cite{shi2023clip} enhancing client feature representation through knowledge distillation and prototype contrastive learning. Some studies explore federated long-tailed learning solutions from a meta-learning perspective~\cite{qian2023long}, aiming to improve model adaptability to new tasks. While these methods have made progress in addressing this problem, some practical data partitioning scenarios are not considered, and there is no specific design to improve algorithm efficiency.

Several approaches that addresses the local long-tailed distribution in centralized  machine learning may have the potential to be integrated with momentum-based FL approaches for long-tailed non-IID data distribution. Researchers have proposed various methods to tackle local long-tailed challenge. Re-balancing strategies such as Focal loss ~\cite{lin2017focal}, PriorCELoss~\cite{hong2021disentangling} and LDAM loss~\cite{cao2019learning} adjust prediction probabilities or introduce label-distribution-aware margin losses to improve tail class performance. However,  experiments will show that the naive integration does not have sufficient effect to address this issue.

\section{Preliminaries}

\subsection{Federated Learning with Momentum}

Federated Learning (FL) allows multiple clients to train a global model collaboratively without sharing their local data. Mathematically, FL optimizes:

\begin{equation}
    \min_{w} \left[ F(w) = \sum_{k=1}^{K} \frac{n_k}{n} F_k(w) \right]
\end{equation}

where \( F_k(w) \) is the local loss function of client \(k\), \( n_k \) is the number of data samples at client \(k\), and \( w \) represents the global model parameters.

FL faces challenges due to data heterogeneity, leading to client drift, where inconsistent model updates slow or prevent convergence. To address this, momentum-based approaches~\cite{xu2021fedcm}~\cite{cheng2023momentum} has been proposed during local training:
\begin{equation}
    \mathbf{v}_{k}^{r} = \alpha \mathbf{g}_{k}^{r} + (1 - \alpha) \Delta_r
\end{equation}
where \(\mathbf{v}_{k}^{r}\) is the momentum of client \(k\) at round \(r\), \(\alpha\) is the momentum value, and \(\Delta_{r}\) is the global momentum obtained from the average of client gradients in the previous round.


\begin{figure}[!t]
	\centering
	\includegraphics[width=1\linewidth]{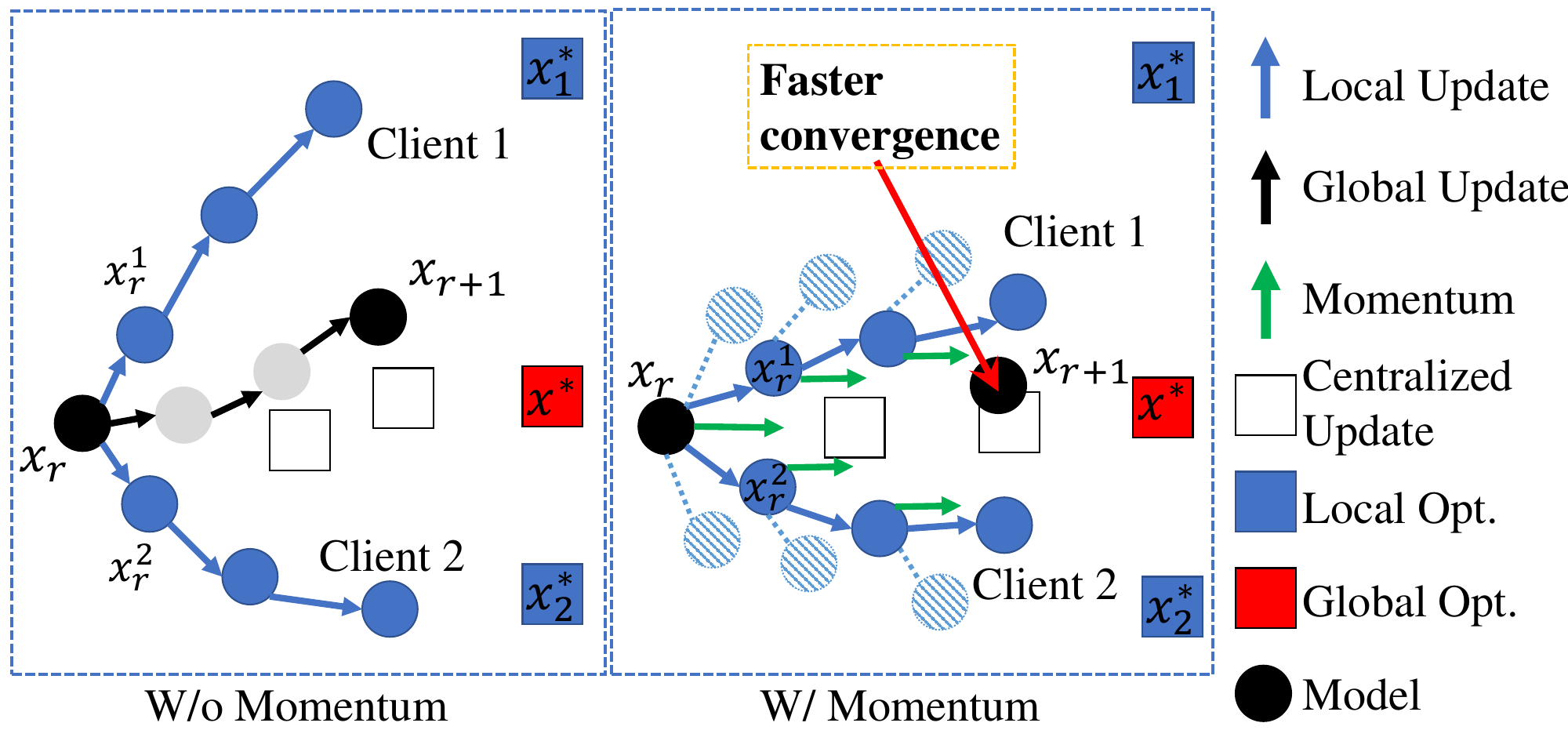}
	\caption{Client drift due to heterogeneity without momentum and alleviation with momentum.}
	\label{fig:client_drift_momentum}
\end{figure}

Figure~1 illustrates how momentum mitigates client drift and accelerates convergence. In detail, the momentum leverages the aggregated gradients from the previous round to obtain global directional information, aligning each client's update direction during local training. This alignment of client gradient updates effectively mitigates client drift.

\subsection{Long-Tailed and non-IID Data Distribution}

 Long-tailed data distributions are practical in real-world scenarios and characterized by a significant imbalance between the most and least frequent classes. We define the imbalance factor IF ~\cite{cao2019learning,shang2022} as:  $IF = \frac{n_1}{n_C}$
where \( n_1 \) and \( n_C \) are the global sample counts of the most and least frequent classes, respectively. When $IF$ is smaller, the tales of the global data distributions is longer.

To introduce non-IID data distribution on the clients, we use Dirichlet allocation~\cite{mcmahan2017communication}: $p_{k,c} \sim \text{Dir}(\beta) $  
where \(p_{k,c}\) is the proportion of samples for class \(c\) allocated to client \(k\), and \(\beta\) controls the degree of heterogeneity. Note that the smaller \(\beta\) denotes higher skew.

\begin{figure}[ht]
	\centering
	\includegraphics[width=1\linewidth]{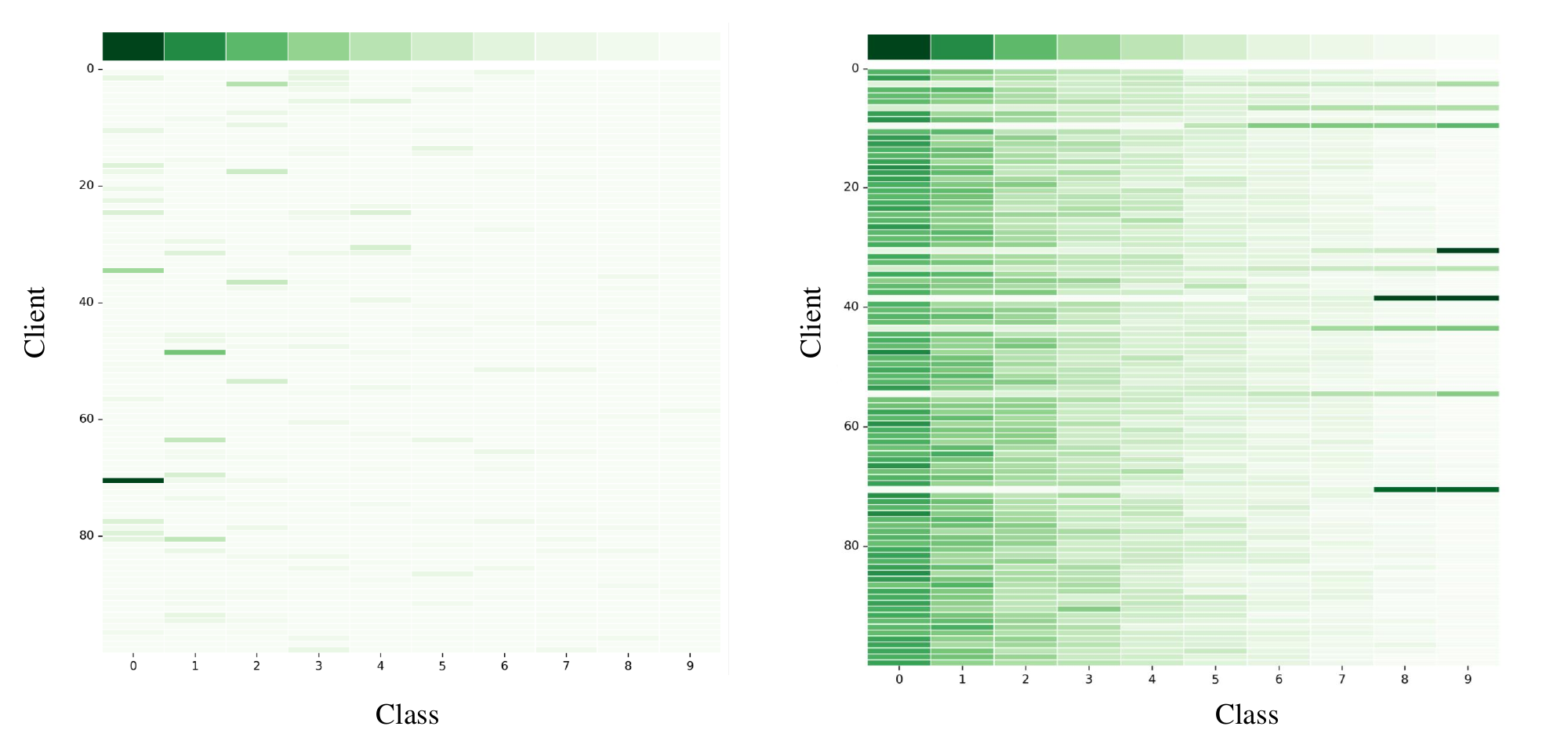}
\caption{ Client data partition on CIFAR-10: FedGrab v.s. ours, when $\beta=0.1$ and $IF=0.1 $.}
	\label{fig:data_distribution_comparison.png}
\end{figure}

While existing works in  LT-problem~\cite{chen2022towards} and FedGrab~\cite{xiao2024fed} leverage a similar long-tailed distributions, they often result in inconsistent client data quantities. That is to say, their data partition naturally incur high skew on data quantities. However, in some practical IoT scenarios~\cite{shuai2022balancefl}, the number of the data samples are similar among different clients. Thus, in this paper, we follow the partition strategy ~\cite{shuai2022balancefl} with $\beta$ and $IF$ to designate the training data to clients, as shown in the right part in Figure~\ref{fig:data_distribution_comparison.png}. The discussion of the data partition shown in the left figure is in Appendix A~\cite{anonymous2025appendix}\footnote{Appendix and full technical proofs are available at \url{https://li-tian-le.github.io/FedWCM-Supplement/}}.

\section{Motivation}

\begin{figure}[!tbp]
    \centering
    \includegraphics[width=1\linewidth]{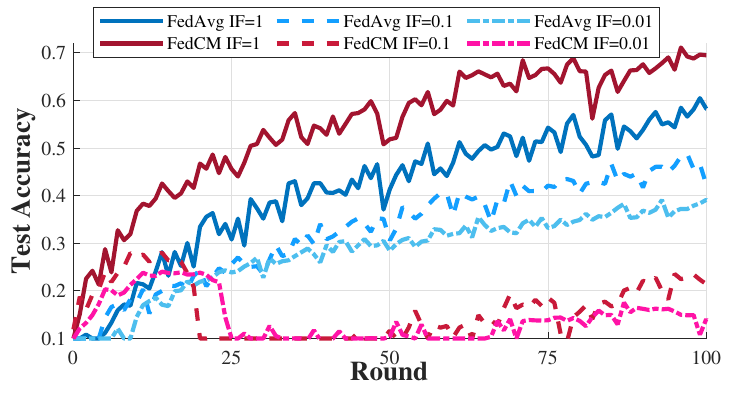}
    \caption{Test accuracy over communication rounds on CIFAR-10 with $\beta=0.1$ and various settings of $IF$.}
    \label{fig:fed_comparison}
\end{figure}

\begin{figure}[ht]
    \centering
    \setlength{\abovecaptionskip}{5pt} 
    \includegraphics[width=0.9\columnwidth]{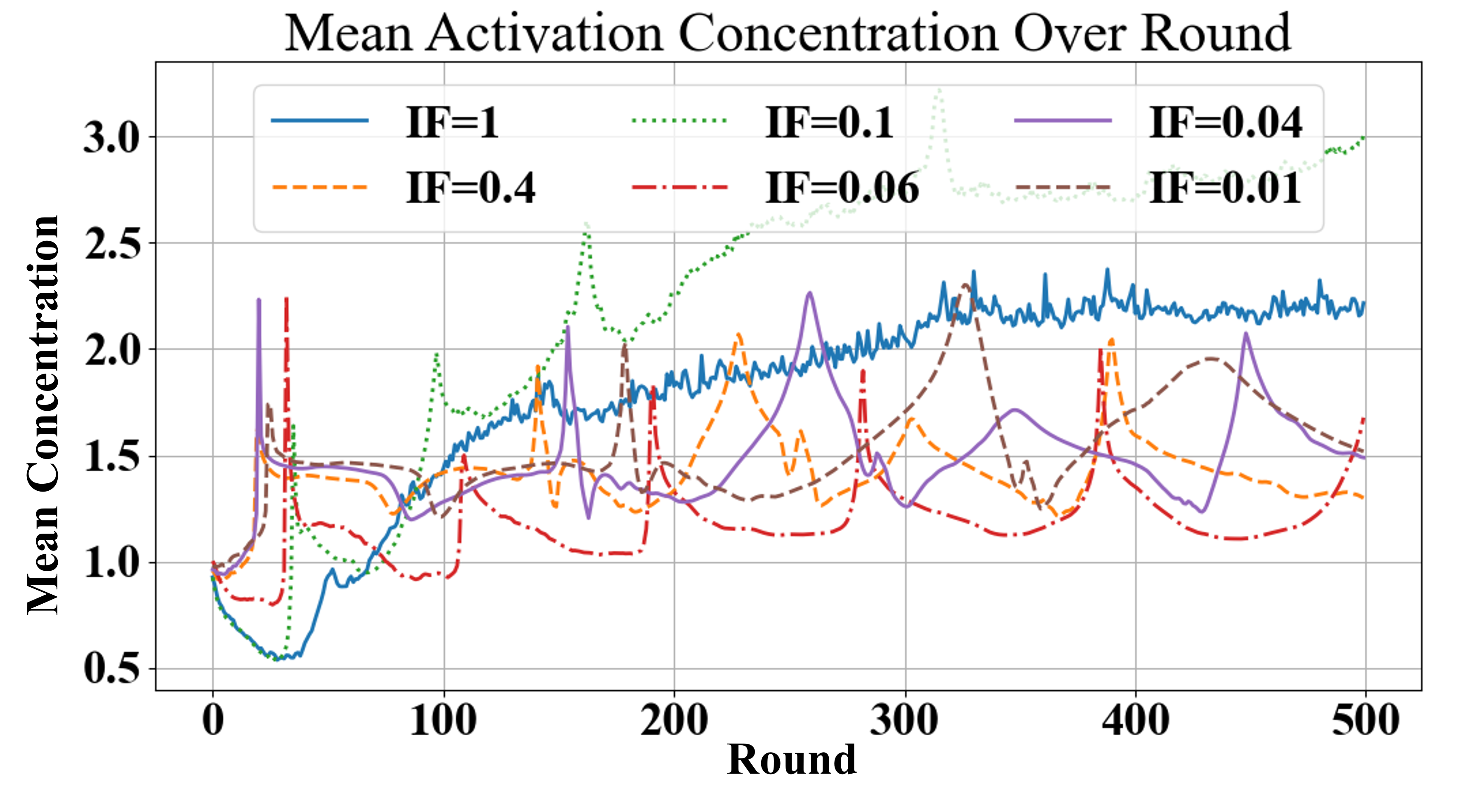} 
    \includegraphics[width=0.9\columnwidth]{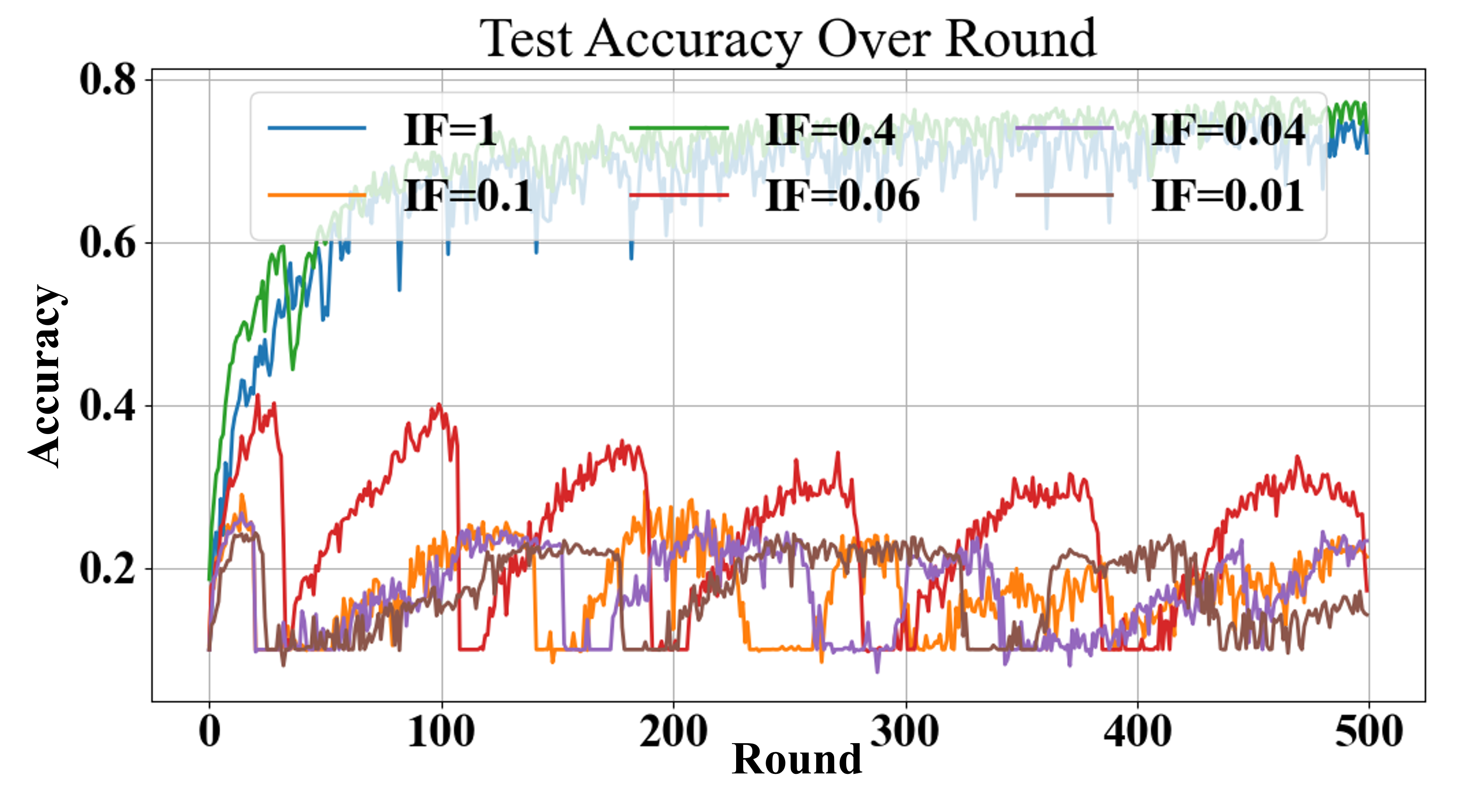} 
    \caption{Both figures show results across six different imbalance factor ($IF$) settings. Top: Average neuron concentration change in FedCM. Bottom: Test accuracy across the same six $IF$ settings.}
    \label{fig:fedcm_critical_analysis}
\end{figure}

Momentum-based methods like FedCM~\cite{xu2021fedcm} introduce global momentum to guide local training and reduce client drift, showing promise for faster convergence and higher accuracy compared to FedAvg under certain conditions (e.g., Dirichlet distribution with $\beta=0.1$ and $IF=1$, as shown in Figure~\ref{fig:fed_comparison}). However, while momentum can align aggregated gradients with the global objective in class-balanced scenarios, it becomes a double-edged sword in long-tailed data distributions, where it amplifies majority-class gradients, exacerbates data imbalance, and risks non-convergence. This issue is evident in the red and pink lines of Figure~\ref{fig:fed_comparison}, which show FedCM failing to converge under $IF=0.1$ and $IF=0.01$.



As shown in Figure~\ref{fig:fedcm_critical_analysis}, the top plot illustrates the mean activation concentration under different imbalance factor (IF) settings. Under balanced conditions (e.g., $IF=1$), the neuron concentration increases steadily, aligning with the theory of neural collapse~\cite{fang2021exploring, kothapalli2022neural, yang2022inducing}. However, under long-tailed distributions, we observe abrupt spikes in neuron concentration at certain critical points. Notably, as the imbalance increases (i.e., as IF decreases), these spikes occur more frequently and violently.

Through detailed analysis in Appendix B~\cite{anonymous2025appendix}, we infer that these sudden increases are caused by the momentum mechanism~\cite{tang2020long}, which leads to dominant neurons corresponding to majority classes occupying the representational space of others. These spikes are not random fluctuations but structured transitions in the optimization dynamics, where momentum amplifies class-specific gradients in a biased way. This results in a phenomenon known as \textbf{Minority Collapse}~\cite{fang2021exploring}, where the network overfits to majority classes, causing the representation of minority classes to degrade sharply.

In this situation, the model's capacity to accurately classify minority classes diminishes significantly, leading to a steep drop in test accuracy. This effect is illustrated in the bottom plot of Figure~\ref{fig:fedcm_critical_analysis}, showing a synchronous decline in accuracy as the minority representations collapse. These insights underline the importance of mitigating momentum-induced instability, which motivates the design of our proposed method.

\section{Methodology}
We propose FedWCM to tackle non-convergence in momentum-based federated learning with long-tailed distributions. FedWCM dynamically adjusts client weights and momentum, reducing the dominance of majority-class clients and enhancing the impact of minority-class clients while modulating momentum to maintain stable convergence.


\subsection{Global Information Gathering}
\begin{figure}[ht]
    \centering
    \includegraphics[width=1\linewidth]{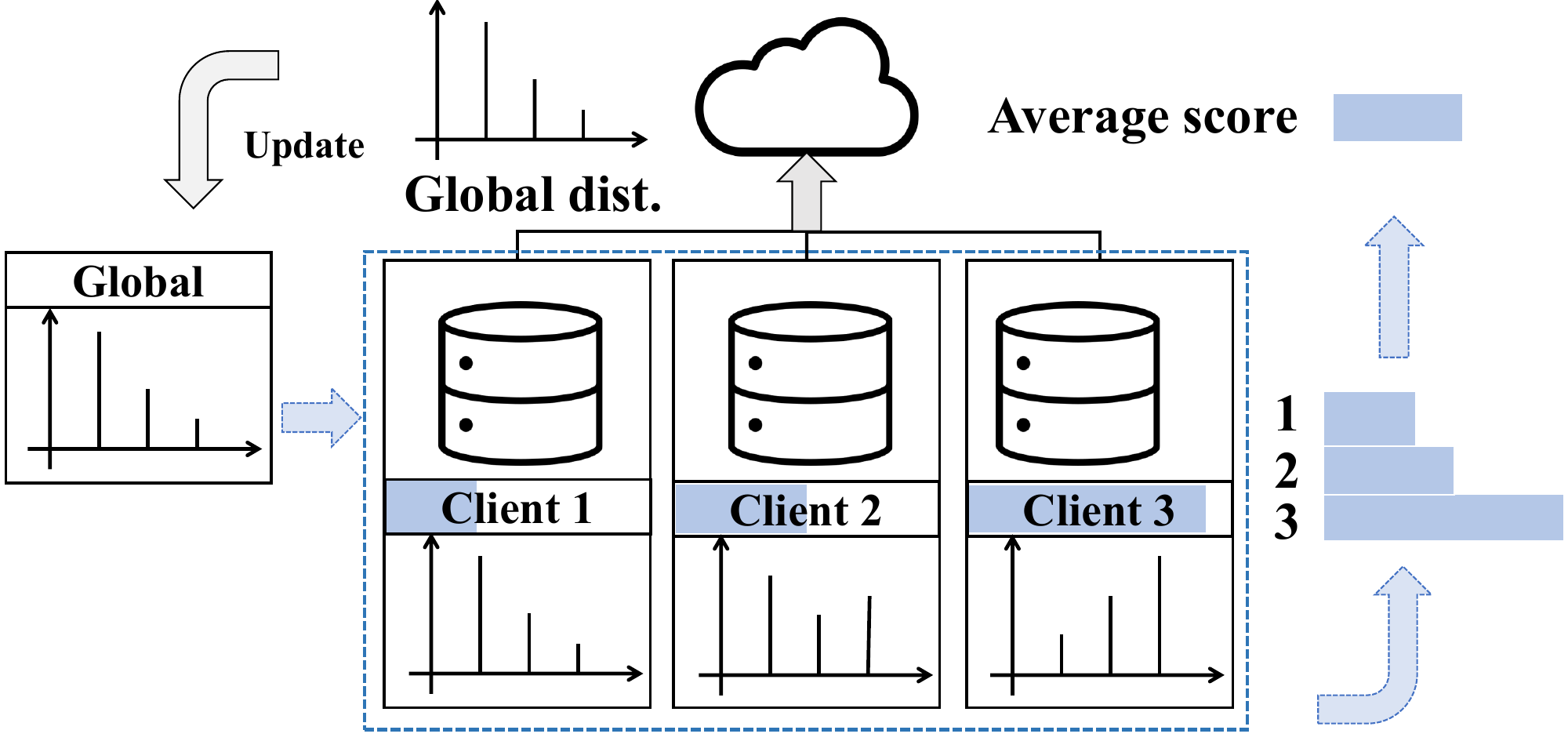}
    \caption{Illustration of global information gathering.}
    \label{fig:global_info_gathering}
\end{figure}
Clients need global knowledge of the data distribution to adjust their local updates. As illustrated in Figure~\ref{fig:global_info_gathering}, we compute and distribute the global data distribution $D_g$ to all clients. Each client computes a score based on the deviation of its local distribution from the global target distribution. The score for client $k$ is:

\begin{equation}
s_k = \frac{\sum_{c=1}^{C} |\hat{p}_c - p_c| \cdot n_{k,c}}{\sum_{c=1}^{C} n_{k,c}}
\end{equation}

where $\hat{p}_c$ is the proportion of class $c$ in the global target distribution, $p_c$ is the proportion of class $c$ in the global distribution, and $n_{k,c}$ is the number of samples of class $c$ for client $k$. By default, the global target distribution is assumed to be uniform, but users can adjust it based on the prior distribution relevant to their specific application scenarios.

A higher score indicates that the client has more globally scarce data. This scoring helps clients adjust their updates with a global perspective, mitigating the impact of data heterogeneity.

\subsection{Parameter Computation}
Next, client weights and adaptive momentum values are computed. As illustrated in Figure~\ref{fig:parameter_computation}, client weights are computed using a modified Softmax function:

\begin{equation}
w_k^r = \frac{\exp(s_k^r / T)}{\sum_{j \in \mathcal{P}_r} \exp(s_j^r / T)}
\end{equation}

where $T$ is a temperature parameter that works inversely with the imbalance: as the data imbalance increases, the temperature decreases, which increases the differences between client weights. This allows clients with more representative data to have more influence during the training process. Conversely, when the global data imbalance is low, the temperature increases, resulting in more uniform client weights. In practice, \(T\) is computed based on the discrepancy between the target (usually uniform) distribution and the actual global data distribution, scaled appropriately by the number of classes to maintain consistent sensitivity across different datasets. This mechanism helps balance the influence of each client on the model during training.

The adaptive momentum value $\alpha_{r+1}$ is calculated as:

\begin{equation}
\alpha_{r+1} = 0.1 + 0.9 \cdot (1 - e^{-\|\mathbf{T/K}\|_1}) \cdot q_r
\end{equation}

where \(q_r\) is the ratio between the average client score sampled in the current round \(\hat{s}^r\) and the overall average score \(\hat{s}\). The base momentum \(\alpha_0\) is set to 0.1. The ratio \(q_r\) reflects the degree to which minority classes are represented in the current round's sampled clients relative to the entire client population. When the current round's average client score is relatively high, indicating better minority class representation, the momentum is increased to leverage more informative gradient updates. This ensures that momentum is dynamically adjusted based on the data imbalance across rounds.

\begin{figure}[ht]
    \centering
    \includegraphics[width=1\linewidth]{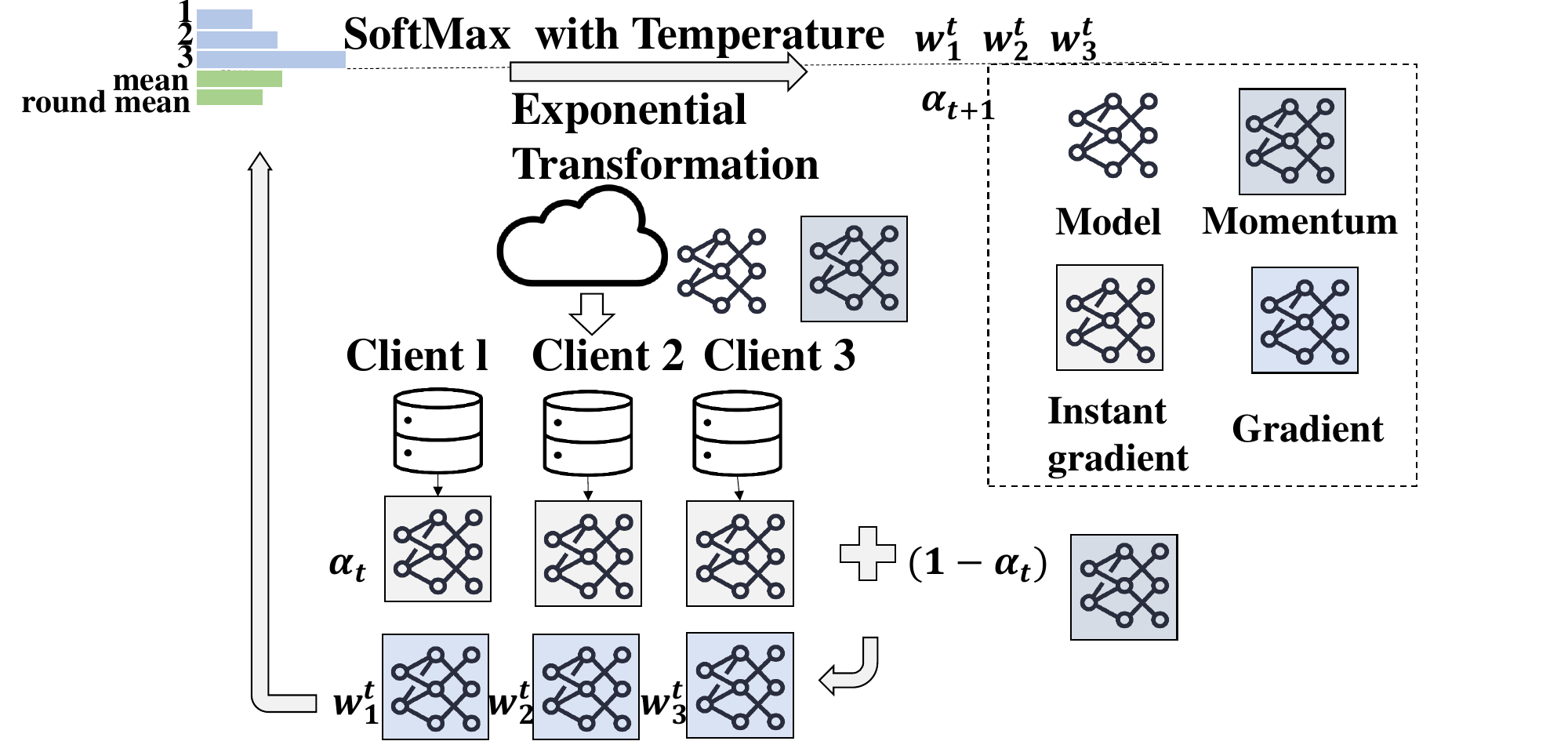}
    \caption{Parameter computation and application.}
    \label{fig:parameter_computation}
\end{figure}

\subsection{Parameter Application in Training}
The computed parameters are applied to local model updates during each round. Clients update their models using the formula:
\begin{equation}
x_{b+1,k}^{r} = x_{b,k}^r - \eta_l (\alpha_r \nabla F(x_{b,k}^r; D_{b,k}) + (1 - \alpha_r)\Delta_r)
\end{equation}

where $\alpha_r$ controls the impact of the current round's gradient. This dynamic adjustment helps balance the effects of majority-class and minority-class clients, improving model robustness across rounds.

\subsection{Overall FedWCM Algorithm}

In summary, we obtain global data distribution to dynamically set parameters for each round's momentum aggregation and application process, enabling momentum to alleviate, rather than exacerbate, the negative impacts of data imbalance. Through integrated mechanisms, our method effectively addresses the challenges posed by long-tail data distributions, ensuring improved federated learning outcomes. The pseudo algorithm is shown in Algorithm~\ref{alg:main_alg}.

\begin{algorithm}[ht]
\caption{FedWCM Algorithm}
\begin{algorithmic}
\State \textbf{Require:} initial model $x_0$, global momentum $\Delta_0$, $\alpha_0 = 0.1$, learning rates $\eta_l, \eta_g$, number of rounds $R$, local iterations $B$
\State Compute $\{s_k\}$ with $D_g$ using Equation (3)
\For{$r = 0$ to $R - 1$}
    \State Sample subset $\mathcal{P}_r$ of clients
    \For{Each client $k \in \mathcal{P}_r$}
        \State $x_{0,k}^r = x_r$
        \For{$b = 0$ to $B - 1$}
            \State Compute $g_{b,k}^r = \nabla f_k(x_{b,k}^r,D_{b,k})$
            \State $v_{b,k}^r = \alpha_r g_{b,k}^r + (1 - \alpha_r)\Delta_r$
            \State $x_{b+1,k}^r = x_{b,k}^r - \eta_l v_{b,k}^r$
        \EndFor
        \State $\Delta_k^r = x_{B,k}^r - x_r$
    \EndFor
    \State Compute $w_k^r$ using Equation (4)
    \State Compute $\alpha_{r+1}$ using Equation (5)
    \State $\Delta_{r+1} = \frac{1}{\eta_l B} \sum_{k \in \mathcal{P}_r} w_k^r\Delta_k^r$
    \State $x_{r+1} = x_r - \eta_g \Delta_{r+1}$
\EndFor
\end{algorithmic}
\label{alg:main_alg}
\end{algorithm}

\subsection{Privacy Discussion}

Note that our approach requires access to the global data distribution. While some works~\cite{zhang2021dubhe,fu2021cic,duan2023federated,liu2023fedlpa} emphasize protecting local data distributions, others~\cite{duan2019astraea,duan2020self,mhaisen2021optimal} tolerate the potential leakage of local data distributions for algorithmic design purposes. Nevertheless, our method is compatible with existing privacy-preserving techniques, such as Homomorphic Encryption (HE), to safeguard local data distributions. Following the approach in BatchCrypt~\cite{254465}, the HE-based protocol proceeds as follows:

\begin{itemize}
    \item \textbf{Key Generation:} A subset of randomly selected clients generates public-private key pairs and distributes the public keys to the other clients.
    \item \textbf{Encryption and Upload:} Each client encrypts their local class distribution using the received public key and uploads the encrypted data to the server.
    \item \textbf{Aggregation:} The server aggregates the encrypted class distributions using homomorphic operations, without decrypting them.
    \item \textbf{Decryption and Reconstruction:} The aggregated ciphertext is sent to the corresponding key generator, who decrypts it to obtain the global class distribution and uploads it to the server.
\end{itemize}

This protocol assumes a semi-honest server (i.e., honest-but-curious) and does not require any trusted third-party entities. We further explore this issue in Appendix~C~\cite{anonymous2025appendix}.

\subsection{Method Generalization}

The original FedWCM algorithm assumes that clients possess roughly equal data volumes. However, in real-world scenarios, this assumption often fails due to data quantity heterogeneity. To address this, we extend FedWCM to better handle high quantity skew, resulting in a generalized version termed \textbf{FedWCM-X}, as shown in Algorithm~\ref{alg:main-X_alg}.

\begin{algorithm}[ht]
\caption{FedWCM-X Algorithm}
\begin{algorithmic}
\State \textbf{Require:} initial model $x_0$, global momentum $\Delta_0$, $\alpha_0 = 0.1$, learning rates $\eta_l, \eta_g$, number of rounds $R$, local iterations $B$, standard iterations $\hat{B}$
\State Compute $\{s_k\}$ with $D_g$ using Equation (3)
\For{$r = 0$ to $R - 1$}
    \State Sample subset $\mathcal{P}_r$ of clients
    \For{Each client $k \in \mathcal{P}_r$}
        \State $x_{0,k}^r = x_r$
        \State $ \eta_l' = \eta_l \cdot \frac{\hat{B}}{B_k} $
        \For{$b = 0$ to $B_k - 1$}
            \State Compute $g_{b,k}^r = \nabla f_k(x_{b,k}^r,D_{b,k})$
            \State $v_{b,k}^r = \alpha_r g_{b,k}^r + (1 - \alpha_r)\Delta_r$
            \State $x_{b+1,k}^r = x_{b,k}^r - \eta_l' v_{b,k}^r$
        \EndFor
        \State $\Delta_k^r = x_{B_k,k}^r - x_r$
    \EndFor
    \State Compute $w_k^r$ using Equation (4)
    \State $ w_k'^r = w_k^r \cdot \frac{n_k}{\sum_{j} n_j} $
    \State Compute $\alpha_{r+1}$ using Equation (5)
    \State $\Delta_{r+1} = \frac{1}{\eta_l \hat{B}} \sum_{k \in \mathcal{P}_r} w_k'^r\Delta_k^r$
    \State $x_{r+1} = x_r - \eta_g \Delta_{r+1}$
\EndFor
\end{algorithmic}
\label{alg:main-X_alg}
\end{algorithm}

FedWCM-X introduces two key modifications to enhance robustness under non-uniform data distribution.

First, we incorporate an additional weighting factor based on each client's data volume. Specifically, for a client \(k\), if its original update weight is \(w_k\), the adjusted weight is defined as:
\[
w_k' = w_k \cdot \frac{n_k}{\sum_j n_j}
\]
where \(n_k\) denotes the number of local data samples held by client \(k\). This adjustment ensures that clients with more data exert proportionally more influence during model aggregation.

Second, we normalize the local learning rate according to the number of mini-batches each client executes. Since clients with more data may perform more local updates, we rescale the learning rate \(\eta_l\) to prevent instability caused by excessively large gradient steps. The adjusted learning rate is given by:
\[
\eta_l' = \eta_l \cdot \frac{\hat{B}}{B_k}
\]
where \(B_k\) is the number of local batches for client \(k\), and \(\hat{B}\) is a reference batch count corresponding to a balanced data split across all clients.

Together, these modifications enable FedWCM-X to retain the benefits of momentum-based correction while adapting to imbalanced client participation, improving convergence stability and model performance in practical federated environments. The effectiveness of FedWCM-X under non-uniform data quantity settings is demonstrated through experiments presented in Appendix~A~\cite{anonymous2025appendix}.

\section{ Convergence Analysis}
\label{sec:convergence}

In this section, we prove that the convergence rate of FedWCM is the same as FedAvg-M~\cite{cheng2023momentum}, e.g., $\sqrt{\frac{L \Delta \sigma^2}{NKR}} + \frac{L \Delta}{R}$.

\begin{assumption}[Standard Smoothness]
\label{ass:smoothness}
Each local objective function $f_i$ is $L$-smooth, i.e., for any $x, y \in \mathbb{R}^d$ and $1 \leq i \leq N$, we have
\begin{equation}
    \|\nabla f_i(x) - \nabla f_i(y)\| \leq L \|x - y\|.
\end{equation}
\end{assumption}

\begin{assumption}[Stochastic Gradients]
\label{ass:stochastic}
There exists $\sigma \geq 0$ such that for any $x \in \mathbb{R}^d$ and $1 \leq i \leq N$, we have
\begin{equation}
    \mathbb{E}_{\xi_i} [\nabla F(x; \xi_i)] = \nabla f_i(x),
\end{equation}
and
\begin{equation}
    \mathbb{E}_{\xi_i} [\|\nabla F(x; \xi_i) - \nabla f_i(x)\|^2] \leq \sigma^2,
\end{equation}
where $\xi_i \sim \mathcal{D}_i$ are independent and identically distributed.
\end{assumption}

Under the assumptions stated above, we provide the convergence result for FedWCM with adaptive $\beta$ and weighted aggregation. The following theorem summarizes the key convergence guarantee.

\begin{theorem}[Convergence of FedWCM]
\label{thm:convergence}
Let $f(x)$ be the global objective function. Under Assumptions~\ref{ass:smoothness} and~\ref{ass:stochastic}, with $g_0 = 0$, $\beta \leq \sqrt{\frac{NKL\Delta}{\sigma^2 R}}$ for any constant $c \in (0, 1]$, $\gamma = \min \left( \frac{1}{24L}, \frac{\beta}{6L} \right)$, and $\eta KL \lesssim \min \left( 1, \frac{1}{\beta \gamma L R}, \left( \frac{L \Delta}{G_0\beta^3 R} \right)^{1/2}, \frac{1}{(\beta N)^{1/2}}, \frac{1}{(\beta^3 N K)^{1/4}} \right)$, FedWCM achieves the following convergence rate:
\begin{equation}
    \frac{1}{R} \sum_{r=0}^{R-1} \mathbb{E}[\|\nabla f(x_r)\|^2] \lesssim \sqrt{\frac{L \Delta \sigma^2}{NKR}} + \frac{L \Delta}{R},
\end{equation}
where $\Delta = f(x_0) - \min_x f(x)$, and $\lesssim$ absorbs constant numerical factors. Here, $G_0 := \frac{1}{N} \sum_{i=1}^{N} \|\nabla f_i(x_0)\|^2$ represents the average squared norm of the client gradients at the initial point $x_0$.
\end{theorem}

The proof of Theorem~\ref{thm:convergence} extends the standard convergence analysis of FedAvg-M \cite{cheng2023momentum} with two key modifications: an adaptive $\beta$ and weighted aggregation. The parameter $\beta$ is dynamically adjusted during training, constrained within $\left[ 0.1, 1 \right)$, and satisfies $\beta \leq \sqrt{\frac{NKL\Delta}{\sigma^2 R}}$, ensuring stability across different data distributions. Additionally, the adaptive aggregation weights mitigate the bias term $U_r$ by accounting for discrepancies between local and global data distributions. These changes preserve the original convergence guarantees while improving adaptability. Further details, including bias analysis, are provided in Appendix E~\cite{anonymous2025appendix}.
\section{Experiments}


\subsection{Experimental Setup}

\begin{table*}[ht]
    \centering
    \scriptsize 
    \setlength{\tabcolsep}{2pt} 
    \caption{Performance comparison under $\beta=0.6$ and $\beta=0.1$ across different datasets and imbalance factors (IF). We report the mean test accuracy under 3 trials on different random seeds.}
    \begin{tabular}{c c *{14}{>{\centering\arraybackslash}m{0.055\linewidth}}} 
        \toprule
        \multicolumn{2}{c}{} & \multicolumn{2}{c}{FedAvg} & \multicolumn{2}{c}{BalanceFL} & \multicolumn{2}{c}{FedCM} & \multicolumn{2}{c}{FedCM + Focal Loss} & \multicolumn{2}{c}{FedCM + Balance Loss} & \multicolumn{2}{c}{FedCM + Balance Sampler} & \multicolumn{2}{c}{FedWCM} \\
        \cmidrule(lr){3-4} \cmidrule(lr){5-6} \cmidrule(lr){7-8} \cmidrule(lr){9-10} \cmidrule(lr){11-12} \cmidrule(lr){13-14} \cmidrule(lr){15-16}
        Dataset & IF & 0.6 & 0.1 & 0.6 & 0.1 & 0.6 & 0.1 & 0.6 & 0.1 & 0.6 & 0.1 & 0.6 & 0.1 & 0.6 & 0.1 \\
        \midrule
        \multirow{5}{*}{Fashion-MNIST} 
        & 1    & \textbf{0.8800} & 0.8074  & 0.8795 & \textbf{0.8443}  & 0.8419 & 0.7604  & 0.8246 & 0.6931  & 0.7451 & 0.6907  & 0.8546 & 0.7821  & 0.8625 & 0.8181 \\
        & 0.5  & \textbf{0.8688} & 0.8079  & 0.8638 & 0.8462  & 0.8601 & \textbf{0.8544}  & 0.8363 & 0.7058  & 0.7622 & 0.6737  & 0.8476 & 0.7906  & 0.8659 & 0.8366 \\
        & 0.1  & 0.8450 & 0.8313  & \textbf{0.8497} & \textbf{0.8475}  & 0.8211 & 0.8268  & 0.8161 & 0.8065  & 0.7873 & 0.8002  & 0.8245 & 0.8252  & 0.8469 & 0.8328 \\
        & 0.05 & 0.8318 & 0.8408  & \textbf{0.8520} & \textbf{0.8545}  & \underline{0.3914} & \underline{0.4975}  & \underline{0.1945} & \underline{0.1967}  & \underline{0.4335} & \underline{0.4064}  & \underline{0.5474} & \underline{0.5273}  & 0.8499 & 0.8426 \\
        & 0.01 & 0.7871 & 0.7894  & \textbf{0.8192} & \textbf{0.8126}  & 0.7378 & 0.7524  & 0.7188 & 0.7265  & 0.8039 & 0.8011  & 0.8027 & 0.8030  & 0.7882 & 0.7947 \\
        \midrule
        \multirow{5}{*}{SVHN} 
        & 1    & 0.9361 & 0.8986  & \textbf{0.9370} & 0.9146  & 0.9246 & 0.8836  & 0.9242 & 0.8749  & 0.8928 & 0.8312  & 0.9310 & 0.8911  & 0.9355 & \textbf{0.9276} \\
        & 0.5  & 0.9251 & 0.9271  & 0.9261 & 0.9243  & 0.7137 & 0.6981  & 0.6068 & 0.5961  & 0.5594 & 0.5870  & 0.7085 & 0.6761  & \textbf{0.9324} & \textbf{0.9284} \\
        & 0.1  & 0.8681 & 0.8741  & 0.8979 & 0.8976  & \underline{0.1594} & \underline{0.0670}  & \underline{0.1959} & \underline{0.1959}  & \underline{0.0762} & \underline{0.1322}  & \underline{0.1959} & \underline{0.1976}  & \textbf{0.9057} & \textbf{0.9024} \\
        & 0.05 & 0.8594 & 0.8647  & 0.8675 & 0.8709  & \underline{0.3100} & \underline{0.0723}  & \underline{0.0670} & \underline{0.1107}  & \underline{0.0969} & \underline{0.0909}  & \underline{0.1802} & \underline{0.0932}  & \textbf{0.8759} & \textbf{0.8836} \\
        & 0.01 & 0.7884 & 0.7803  & 0.7901 & 0.7954  & \underline{0.0670} & \underline{0.0670}  & \underline{0.0670} & \underline{0.0751}  & \underline{0.0760} & \underline{0.0759}  & \underline{0.1736} & \underline{0.2427}  & \textbf{0.7998} & \textbf{0.8408} \\
        \midrule
        \multirow{5}{*}{CIFAR-10} 
        & 1    & 0.7906 & 0.6881  & 0.7629 & 0.6813  & 0.8126 & 0.7092  & 0.8040 & 0.6937  & 0.7931 & 0.7169  & 0.8065 & 0.7198  & \textbf{0.8242} & \textbf{0.7337} \\
        & 0.5  & 0.7535 & 0.7183  & 0.7539 & 0.7429  & 0.6793 & 0.6686  & 0.6565 & 0.6319  & 0.6877 & 0.6924  & 0.6968 & 0.6590  & \textbf{0.7926} & \textbf{0.7968} \\
        & 0.1  & 0.6232 & 0.6775  & 0.6380 & 0.6541  & \underline{0.2175} & \underline{0.2393}  & \underline{0.1311} & \underline{0.3095}  & \underline{0.1864} & \underline{0.3016}  & \underline{0.2871} & \underline{0.3994}  & \textbf{0.6905} & \textbf{0.7207} \\
        & 0.05 & 0.5715 & 0.5642  & 0.5652 & 0.5535  & \underline{0.2274} & \underline{0.2358}  & \underline{0.2005} & \underline{0.1413}  & \underline{0.2680} & \underline{0.2525}  & \underline{0.1427} & \underline{0.1315}  & \textbf{0.6006} & \textbf{0.6132} \\
        & 0.01 & 0.4567 & 0.4600  & 0.4731 & 0.4616  & \underline{0.1865} & \underline{0.2312}  & \underline{0.1687} & \underline{0.2023}  & \underline{0.2087} & \underline{0.2405}  & \underline{0.1249} & \underline{0.1584}  & \textbf{0.4983} & \textbf{0.5012} \\
        \midrule
        \multirow{5}{*}{CIFAR-100} 
        & 1    & 0.4297 & 0.3731  & 0.3691 & 0.3232  & 0.4129 & \underline{0.2400}  & 0.3990 & \underline{0.2357}  & 0.3630 & \underline{0.2089}  & 0.3599 & \underline{0.2339}  & \textbf{0.4545} & \textbf{0.3858} \\
        & 0.5  & 0.3545 & 0.3882  & 0.3203 & 0.3639  & 0.2996 & 0.4200  & 0.3058 & 0.3853  & 0.2694 & 0.3722  & 0.2835 & 0.3790  & \textbf{0.4195} & \textbf{0.4202} \\
        & 0.1  & 0.2839 & 0.2744  & 0.2440 & 0.2407  & 0.2948 & 0.3135  & 0.3014 & 0.3166  & 0.2952 & 0.3156  & 0.2952 & 0.2955  & \textbf{0.3150} & \textbf{0.3235} \\
        & 0.05 & 0.2155 & 0.2300  & 0.2070 & 0.2157  & \underline{0.1130} & 0.2695  & \underline{0.0100} & 0.2806  & \underline{0.1000} & 0.2786  & \underline{0.0930} & 0.2721  & \textbf{0.2573} & \textbf{0.2832} \\
        & 0.01 & 0.1663 & 0.1885  & 0.1565 & 0.1609  & \underline{0.0116} & 0.1035  & \underline{0.0109} & 0.1027  & \underline{0.0100} & 0.1286  & \underline{0.0100} & 0.0723  & \textbf{0.1985} & \textbf{0.2005} \\
        \midrule
        \multirow{5}{*}{ImageNet} 
        & 1    & 0.2760 & 0.2290  & 0.2292 & 0.1947  & 0.2479 & \underline{0.1408}  & 0.2438 & \underline{0.1222}  & 0.2082 & \underline{0.1024}  & 0.2134 & \underline{0.1155}  & \textbf{0.3094} & \textbf{0.2462} \\
        & 0.5  & 0.2154 & 0.2140  & 0.1628 & 0.2124  & \underline{0.1045} & \underline{0.0392}  & \underline{0.0923} & \underline{0.0695}  & \underline{0.0928} & \underline{0.0544}  & \underline{0.1154} & \underline{0.1067}  & \textbf{0.2598} & \textbf{0.2198} \\
        & 0.1  & 0.1631 & 0.1535  & 0.1124 & 0.1161  & 0.1796 & 0.1738  & 0.1864 & 0.1763  & 0.1796 & 0.1788  & 0.1528 & 0.1521  & \textbf{0.1923} & \textbf{0.1874} \\
        & 0.05 & 0.1458 & 0.1355  & 0.0915 &  0.0998  & \underline{0.0052} & 0.1597  & 0.1355 & 0.1448  & 0.1471 & 0.1576  & 0.1130 & 0.1542  & \textbf{0.1626} & \textbf{0.1660} \\
        & 0.01 & 0.0882 & 0.1123  & 0.0627 & 0.0612  & \underline{0.0050} &  0.1137   & \underline{0.0063} & 0.1354  & \underline{0.0050} & 0.1209  & \underline{0.0052} & 0.1217  & \textbf{0.0974} & \textbf{0.1383} \\
        \bottomrule
    \end{tabular}
    \label{tab:performance_comparison}
\end{table*}

We mainly conduct experiments on Fashion-MNIST~\cite{xiao2017fashion}, SVHN~\cite{netzer2011reading}, CIFAR-10~\cite{krizhevsky2009learning}, CIFAR-100~\cite{krizhevsky2009learning} and ImageNet~\cite{deng2009imagenet} datasets. For the data partition, we adopt the partition strategy as shown in the preliminary part. By default, we set $p_{k,c} \sim \text{Dir}(\beta)$, where $\beta = 0.1$. For the imbalanced ratio $IF$, the default setting is $IF = 0.1$. Note that smaller $\beta$ denotes worse skews while smaller $IF$ denotes higher class imbalance level.


By default, we use a multilayer perceptron (MLP) architecture for the Fashion-MNIST dataset. For the SVHN and CIFAR-10 datasets, we employ ResNet-18~\cite{he2016deep} as the backbone network. For the CIFAR-100 and ImageNet datasets, we use ResNet-34~\cite{he2016deep}. We set the batch size to 50, the local learning rate $\eta_l$ to 0.1, the global learning rate $\eta_g$ to 1, and the local epoch to 5. By default, the number of clients is set to 100 with a participation rate of 0.1 and 500 communication rounds for each experiment. For CIFAR-100 and ImageNet, the default number of clients is set to 40, and we conduct 300 communication rounds. Additional experimental settings can be found in the corresponding figures/tables. Note that all experiments were implemented in PyTorch and conducted on a workstation equipped with four NVIDIA GeForce RTX 3090 GPUs.

\subsection{Overall Accuracy Evaluation}

In this experiment, we compare FedWCM against several representative baselines to evaluate its effectiveness in handling long-tailed and heterogeneous federated learning scenarios. We consider two main categories of methods: (1) standard and long-tail-specific federated learning methods, including FedAvg~\cite{mcmahan2017communication}, BalanceFL~\cite{shuai2022balancefl}, and FedGrab~\cite{xiao2024fed}; (2) improved variants of FedCM incorporating common imbalance handling techniques such as Focal Loss~\cite{lin2017focal}, Balance Loss~\cite{hong2021disentangling}, and Balance Sampler~\cite{he2009learning}. Experiments are conducted on Fashion-MNIST, SVHN, CIFAR-10, CIFAR-100, and ImageNet, under Dirichlet distributions with $\beta=0.1$ and $\beta=0.6$ to simulate varying degrees of non-IID data.

Table~\ref{tab:performance_comparison} summarizes the results on CIFAR-10. Across almost all settings, FedWCM achieves the highest accuracy, demonstrating strong generalization ability and robustness under both mild and severe imbalance. For instance, at $\beta=0.6$, $IF=0.1$, FedWCM reaches 0.6905 accuracy, outperforming FedAvg (0.6232), BalanceFL (0.638), and FedGrab (0.326). The performance gap becomes more pronounced as imbalance increases.

While FedGrab performs competitively under moderate heterogeneity (e.g., $IF=1$, $IF=0.5$), it degrades significantly in highly imbalanced or heterogeneous settings. Particularly under $\beta=0.1$, its accuracy drops sharply — for example, only 32.60\% at $IF=0.1$, compared to FedAvg's 67.75\% and FedWCM's 72.07\%. This highlights FedGrab's sensitivity to severe non-IID distributions and limited robustness.

FedAvg, while not specifically designed for long-tail data, provides a stable baseline across most scenarios, though it generally underperforms compared to more specialized methods. BalanceFL shows slight improvements over FedAvg, but its performance remains inconsistent.

As for FedCM and its variants, these methods generally fail to converge or produce very low accuracy in long-tailed settings. FedCM alone yields only 0.2175 accuracy in some cases. Even with the introduction of Focal Loss, Balance Loss, or Balance Sampler, accuracy remains low (e.g., 0.1311, 0.1864, and 0.2871, respectively), indicating that these improvements are insufficient to resolve convergence issues in the presence of severe imbalance.

It is also worth noting that the effectiveness of FedWCM varies depending on the dataset and model complexity. On Fashion-MNIST, the simpler 3-layer MLP architecture limits the benefits of momentum-based strategies. Moreover, datasets with more classes (e.g., CIFAR-100 and ImageNet) show slightly reduced performance gains for FedWCM due to diluted long-tail effects across many classes.

In summary, FedWCM consistently outperforms existing methods, including FedGrab and FedAvg, particularly under harsh conditions. Its adaptive momentum and weight correction mechanism ensure stable and accurate learning in federated environments with long-tailed, heterogeneous data.

\begin{table}[t]
    \centering
    \footnotesize
    \setlength{\tabcolsep}{3pt}
    \caption{Performance comparison on CIFAR-10 under $\beta=0.6$ and $\beta=0.1$ with varying imbalance factors (IF). Results are averaged over 3 runs with different seeds.}
    \begin{tabular}{c c *{6}{>{\centering\arraybackslash}m{0.09\linewidth}}}
        \toprule
        \multicolumn{2}{c}{} & \multicolumn{2}{c}{FedAvg} & \multicolumn{2}{c}{FedGrab} & \multicolumn{2}{c}{FedWCM} \\
        \cmidrule(lr){3-4} \cmidrule(lr){5-6} \cmidrule(lr){7-8}
        Dataset & IF & 0.6 & 0.1 & 0.6 & 0.1 & 0.6 & 0.1 \\
        \midrule
        \multirow{5}{*}{CIFAR-10} 
        & 1    & 0.7906 & 0.6881  & 0.7950 & 0.6813  & \textbf{0.8242} & \textbf{0.7337} \\
        & 0.5  & 0.7535 & 0.7183  & 0.7810 & 0.6560  & \textbf{0.7926} & \textbf{0.7968} \\
        & 0.1  & 0.6232 & 0.6775  & 0.6880 & 0.3260  & \textbf{0.6905} & \textbf{0.7207} \\
        & 0.05 & 0.5715 & 0.5642  & 0.5000 & 0.1870  & \textbf{0.6006} & \textbf{0.6132} \\
        & 0.01 & 0.4567 & 0.4600  & 0.3140 & 0.1350  & \textbf{0.4983} & \textbf{0.5012} \\
        \bottomrule
    \end{tabular}
    \label{tab:hyperparams}
\end{table}


\begin{figure}[ht]
	\centering
	\includegraphics[width=1\linewidth]{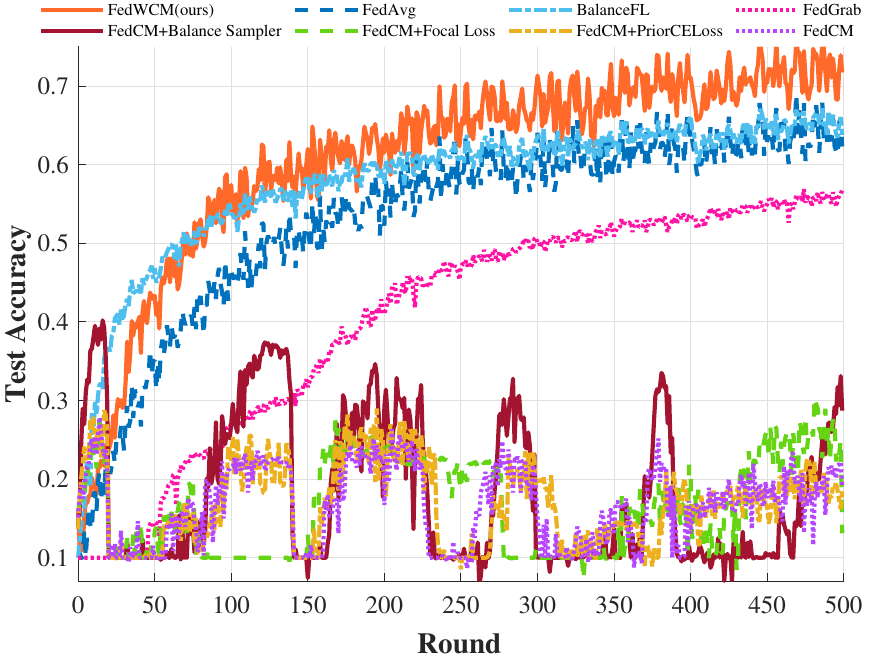}
	\caption{Test accuracy on various methods ($\beta=0.6$, $IF=0.1$).}
	\label{fig:comparison_IF0.1_Dir0.6}
\end{figure}

\subsection{Efficiency Evaluation of FedWCM}
\paragraph{Convergence experiments.}
As illustrated in Figure \ref{fig:comparison_IF0.1_Dir0.6}, under the settings of $\beta= 0.6$ and $IF = 0.1$, FedWCM in the orange line demonstrates a very rapid convergence in test accuracy during the early stages, quickly achieving a high accuracy level. FedWCM’s performance continues to improve, maintaining a high accuracy range, and eventually stabilizing around 400 iterations, reaching a test accuracy of 75\%. Compared to other methods such as FedAvg and FedGrab, FedWCM consistently maintains a leading position throughout most of the iterations. Conversely, different variants of FedCM exhibit characteristics of non-convergence. 

Besides, compared to most algorithms, FedWCM demonstrates a rapid increase in test accuracy during the initial stages of iteration, surpassing the 60\% accuracy threshold at around 120 iterations. In comparison, BalanceFL approaches this level at around 200 iterations, while FedAvg requires approximately 300 iterations to reach the same accuracy. In contrast, FedGrab fails to achieve the 60\% accuracy threshold even after more than 500 iterations. Additionally, the convergence speed of FedWCM is significantly faster than that of other converging algorithms.

\paragraph{Per-label accuracy pattern.}
As shown in Figure \ref{fig:per_class_accuracy}, with $IF=0.1$ and $\beta=0.6$, FedWCM achieves high accuracy on minor labels, notably outperforming FedAvg and FedCM on labels 6, 7, 8, and 9. Figure 2 shows the global label distribution, where label 1 is the most frequent and label 9 the least. This highlights FedWCM's strength in handling long-tailed distributions. In contrast, FedCM's accuracy drops sharply with decreasing label frequency, nearing zero for label 9.

\begin{figure}[!t]
\centering
\includegraphics[width=1\linewidth]{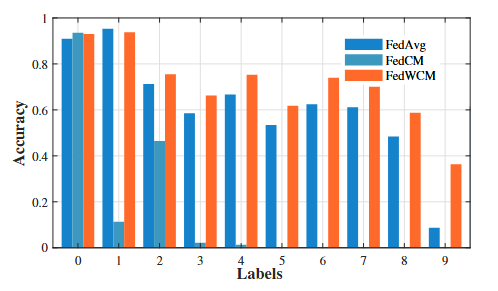}
\caption{Per-label accuracy comparison of various methods when $\beta=0.6$ and $IF=0.1$.}
\label{fig:per_class_accuracy}
\end{figure}

\subsection{Scalability Analysis}

\paragraph{Client Participation Rate.}

FedWCM maintains relatively high accuracy across all levels of client participation rates shown in Table \ref{tab:participation_rate_comparison}, particularly at lower participation rates (i.e, 5\% and 10\%), where its accuracy significantly surpasses that of FedAvg and FedCM. Although the accuracy in FedWCM gradually decreases as the participation rate increases (i.e., from 0.7127 at 5\% to 0.6025 at 80\%), the decline is more gradual compared to FedAvg and FedCM.

\begin{table}[!t]
\centering
\caption{Comparison under different client smapling rates.}
\label{tab:participation_rate_comparison}
\begin{tabular}{cccc}
\toprule
\textbf{Sampling Rate} & \textbf{FedAvg} & \textbf{FedCM} & \textbf{FedWCM} \\
\midrule
5\% & 0.6865 & \underline{0.3130} & \textbf{0.7127} \\
10\% & 0.6232 & \underline{0.1918} & \textbf{0.6905} \\
20\% & 0.6450 & \underline{0.3006} & \textbf{0.7164} \\
40\% & 0.6418 & \underline{0.2268} & \textbf{0.6933} \\
80\% & 0.6441 & \underline{0.1000} & \textbf{0.6980} \\
\bottomrule
\end{tabular}%
\end{table}

\paragraph{Total Client Number.}
\begin{figure}[!t]
    \centering
    \includegraphics[width=1\linewidth]{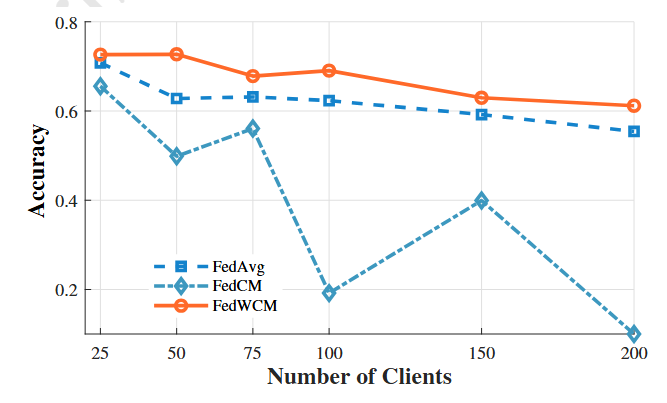}
    \caption{Test accuracy w.r.t the number of clients.}
    \label{fig:client_num_comparison}
\end{figure}

As shown in Figure~\ref{fig:client_num_comparison}, as the number of clients increases, the amount of data allocated to each client decreases, exacerbating data imbalance under the same $IF$, which leads to performance degradation for all three algorithms. Among them, FedWCM exhibits the slowest decline and maintains a relatively high accuracy (around 0.7). In contrast, FedAvg shows a significant performance drop when the number of clients reaches 50, while FedCM suffers from severe performance fluctuations as the number of clients increases, often resulting in non-convergence.



\subsection{Ablation Study}

\paragraph{Local Epochs.}
Figure~\ref{fig:local_rounds_comparison} compares the accuracy of FedAvg, FedCM, and FedWCM across different numbers of local epochs (1, 5, 10, 20 epochs). The results show that FedWCM consistently outperforms the other algorithms across all local round settings, with its accuracy significantly improving as the number of local epochs increases. This performance advantage is particularly evident in mid-to-high round settings. In contrast, FedAvg demonstrates relatively stable performance but consistently falls short of FedWCM, while FedCM displays considerable variability, with accuracy significantly lower than FedWCM across the settings.
\begin{figure}[!t]
    \centering
    \includegraphics[width=1\linewidth]{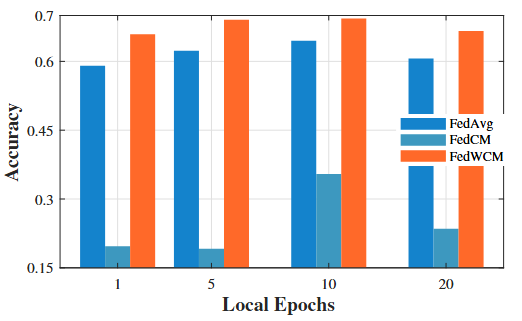}
    \caption{Test accuracy w.r.t local epochs.}
    \label{fig:local_rounds_comparison}
\end{figure}
\paragraph{Various settings of \texorpdfstring{$\beta$}{beta} and \texorpdfstring{IF}{IF}}


Table \ref{tab:comparison_fedavg_fedcm} presents the performance comparison of FedAvg, FedCM, and FedWCM when $\beta = 0.1$ and $\beta = 0.6$. The FedWCM algorithm demonstrates outstanding performance across all configurations, especially when addressing different Dirichlet parameters and varying information factors, even when other algorithms do not converge. FedWCM shows low sensitivity to the $\beta$ parameter, maintaining high-performance levels even under more dispersed data distributions. Furthermore, although all algorithms generally exhibit a performance decline as the information factor decreases, FedWCM experiences a relatively minor drop. Particularly at shallow $IF$ values (i.e., $IF=0.01$), FedWCM remains effective in adapting to sparse data environments, showing its superior adaptability.


\begin{table}[!t]
\centering
\caption{Comparison of various approaches under different settings.}
\label{tab:comparison_fedavg_fedcm}
\resizebox{\columnwidth}{!}{%
\begin{tabular}{cccccccc}
\toprule
\textbf{\(\beta = 0.1\)} & \textbf{IF} & \textbf{1} & \textbf{0.4} & \textbf{0.1} & \textbf{0.06} & \textbf{0.04} & \textbf{0.01} \\
\midrule
FedAvg & & 0.6859 & 0.7059 & 0.6228 & 0.6295 & 0.5358 & 0.4838 \\
FedCM & & 0.7179 & 0.7394 & \underline{0.2346} & \underline{0.2077} & \underline{0.2206} & \underline{0.2283} \\
FedWCM & & \textbf{0.7337} & \textbf{0.7735} & \textbf{0.6629} & \textbf{0.6538} & \textbf{0.5972} & \textbf{0.5078} \\
\midrule
\textbf{\(\beta = 0.6\)} & \textbf{IF} & \textbf{1} & \textbf{0.4} & \textbf{0.1} & \textbf{0.06} & \textbf{0.04} & \textbf{0.01} \\
\midrule
FedAvg & & 0.7912 & 0.7294 & 0.6232 & 0.5801 & 0.5543 & 0.4637 \\
FedCM & & 0.8104 & 0.7363 & \underline{0.1918} & \underline{0.2616} & \underline{0.1894} & \underline{0.2399} \\
FedWCM & & \textbf{0.8426} & \textbf{0.7969} & \textbf{0.6905} & \textbf{0.6216} & \textbf{0.6042} & \textbf{0.5164} \\
\bottomrule
\end{tabular}%
}
\end{table}

\subsection{Supplementary Experiments}

We provide several supplementary experiments~\cite{anonymous2025appendix} to further validate the effectiveness and robustness of our method. Appendix A presents results of FedWCM-X, a generalized extension of FedWCM designed for scenarios with unequal data volumes among clients, demonstrating its stable performance under such imbalance. Appendix B investigates neuron concentration patterns across FedAvg, FedCM, and FedWCM under varying data distributions, aiming to explore the root causes of non-convergence in long-tailed settings. Appendix C examines the resource overhead introduced when integrating FedWCM with homomorphic encryption, showing the approach remains feasible in privacy-sensitive environments. Appendix D further supplements the performance of momentum-based methods against other heterogeneous FL baselines.
\section{Conclusion}

In this work, we design a novel momentum-based federated learning algorithm called FedWCM to address the convergence challenges in long-tailed non-IID data distributions. By dynamically adjusting momentum aggregation and application, FedWCM effectively mitigates the negative impacts of skewed data distributions, resulting in improved convergence and performance across various datasets. Our extensive experiments demonstrate FedWCM's superiority over state-of-the-art algorithms, making it a robust solution for federated learning in complex and imbalanced data scenarios while leverages the advantages of momentum-based approach.



\section*{Acknowledgments}
\sloppy
The research was supported in part by the China NSFC Grant (No. 62372307, No. U2001207, No. 62472366), Guangdong NSF (No. 2024A1515011691), Shenzhen Science and Technology Program (No. RCYX20231211090129039), Shenzhen Science and Technology Foundation (No. JCYJ20230808105906014, No. ZDSYS20190902092853047), Guangdong Provincial Key Lab of Integrated Communication, Sensing and Computation for Ubiquitous Internet of Things (No. 2023B1212010007), 111 Center (No. D25008), and the Project of DEGP (No. 2023KCXTD042, No. 2024GCZX003). Lu WANG is the corresponding author.

\clearpage
\appendix
\section*{Appendix}
\section{Analysis of FedWCM under Different Data Partitioning}
We opted for a custom data partitioning approach to ensure that the data across clients is roughly consistent. Existing long-tailed heterogeneous datasets often present challenges due to extreme data imbalance. To verify the applicability of our method on other datasets, we implemented FedGrab~\cite{guo2022fedgrab}'s data partitioning and conducted comparative experiments.
\subsection{Problem Description}
%
There is currently no universal method for long-tailed heterogeneous partitioning. BalanceFL~\cite{shuai2022balancefl} uses its own long-tailed heterogeneous partitioning, while CLIP~\cite{zhang2023clip2fl} and Creff~\cite{liu2023creff} (using the same partitioning), FedGrab, and our approach first generate a long-tailed dataset, followed by Dirichlet partitioning. However, this methodology has a drawback: when partitioning a long-tailed dataset, the sampling from long-tailed data means that even with Dirichlet sampling, originally majority classes are likely to remain majority classes on clients.

CLIP and Creff, along with FedGrab, provide two partitioning methods, both generating Dirichlet distributions for each class and then allocating to clients. Yet, this can lead to some clients having no data (i.e., proportions in all classes are so low that they do not constitute even one data point, especially when the imbalance factor is small). To address this, the former repeatedly samples until the requirement is met, which may indirectly control the degree of long-tailed distribution, while the latter assigns at least one data point to each client.

%
%
%
%
For our method, the original momentum-based method~\cite{xu2021fedcm,cheng2023momentum} initially did not address the issue of inconsistent data quantity due to two reasons: 1) Solving heterogeneity issues does not necessarily require addressing quantity heterogeneity ~\cite{ye2023heterogeneous}, as they primarily target distributional heterogeneity. 2) The momentum base method introduces a fixed global momentum in a weighted manner each round. Therefore, when there is a large disparity in data quantity between clients, more data leads to more batches. This results in multiple additions of momentum in that client, negating the intended effect of reducing client variance.

Secondly, our main text introduces a method that weights based on data distribution. If we use a common method for addressing data quantity heterogeneity, weighting by data quantity, it may overlap with our method, preventing an effective analysis of our method's effects.

Lastly, our comparison in the main text is also justified, because if a method can address data quantity heterogeneity, it should also perform well in non-heterogeneous scenarios. Here, we supplement experiments with partitioning methods that increase data quantity disparity, demonstrating the applicability of our method under various data distributions.

%
%
%
%
%
%

\begin{algorithm}[ht]
\caption{FedWCM-X Algorithm}
\begin{algorithmic}
\State \textbf{Require:} initial model $x_0$, global momentum $\Delta_0$, $\alpha_0 = 0.1$, learning rates $\eta_l, \eta_g$, number of rounds $R$, local iterations $B$, standard iterations $\hat{B}$
\State Compute $\{s_k\}$ with $D_g$ using Equation (3)
\For{$r = 0$ to $R - 1$}
    \State Sample subset $\mathcal{P}_r$ of clients
    \For{Each client $k \in \mathcal{P}_r$}
        \State $x_{0,k}^r = x_r$
        \State $ \eta_l' = \eta_l \cdot \frac{\hat{B}}{B_k} $
        \For{$b = 0$ to $B_k - 1$}
            \State Compute $g_{b,k}^r = \nabla f_k(x_{b,k}^r,D_{b,k})$
            \State $v_{b,k}^r = \alpha_r g_{b,k}^r + (1 - \alpha_r)\Delta_r$
            \State $x_{b+1,k}^r = x_{b,k}^r - \eta_l' v_{b,k}^r$
        \EndFor
        \State $\Delta_k^r = x_{B_k,k}^r - x_r$
    \EndFor
    \State Compute $w_k^r$ using Equation (4)
    \State $ w_k'^r = w_k^r \cdot \frac{n_k}{\sum_{j} n_j} $
    \State Compute $\alpha_{r+1}$ using Equation (5)
    \State $\Delta_{r+1} = \frac{1}{\eta_l \hat{B}} \sum_{k \in \mathcal{P}_r} w_k'^r\Delta_k^r$
    \State $x_{r+1} = x_r - \eta_g \Delta_{r+1}$
\EndFor
\end{algorithmic}
\label{alg:main-X_alg}
\end{algorithm}
\subsection{Method Generalization}

Our original method assumed that clients have similar amounts of data. To address scenarios with significant disparities in client data quantities, we proposed an improved FedWCM for high quantity skew, naming FedWCM-X, as shown in Algorithm \ref{alg:main-X_alg}. In detail, we outline the two steps required for this extension:

1. Building upon the existing weighting, we additionally weight by data quantity. Specifically, if the current round's weight is $w_k$, we now multiply it by $\frac{n_k}{\sum_{j} n_j}$, where $n_k$ represents the data quantity of the $k$-th client.

$$ w_k' = w_k \cdot \frac{n_k}{\sum_{j} n_j} $$

2. We adjust the learning rate $\eta_l$ based on the batch numbers corresponding to different data quantities. This involves dividing $\eta_l$ by the current batch number $B_k$, and then multiplying by a standard batch number $\hat{B}$, which is the number of batches a client would have if the data were evenly distributed across clients.

$$ \eta_l' = \eta_l \cdot \frac{\hat{B}}{B_k} $$

\subsection{Experimental Section}

We illustrate the superiority of our method over six other approaches through an accuracy variation graph on this dataset( as shown in Figure ~\ref{fig:fedgrab_data_distribution}), using the partitioning strategy based on FedGrab~\cite{guo2022fedgrab}.
\begin{figure}[ht]
\centering
\includegraphics[width=\columnwidth]{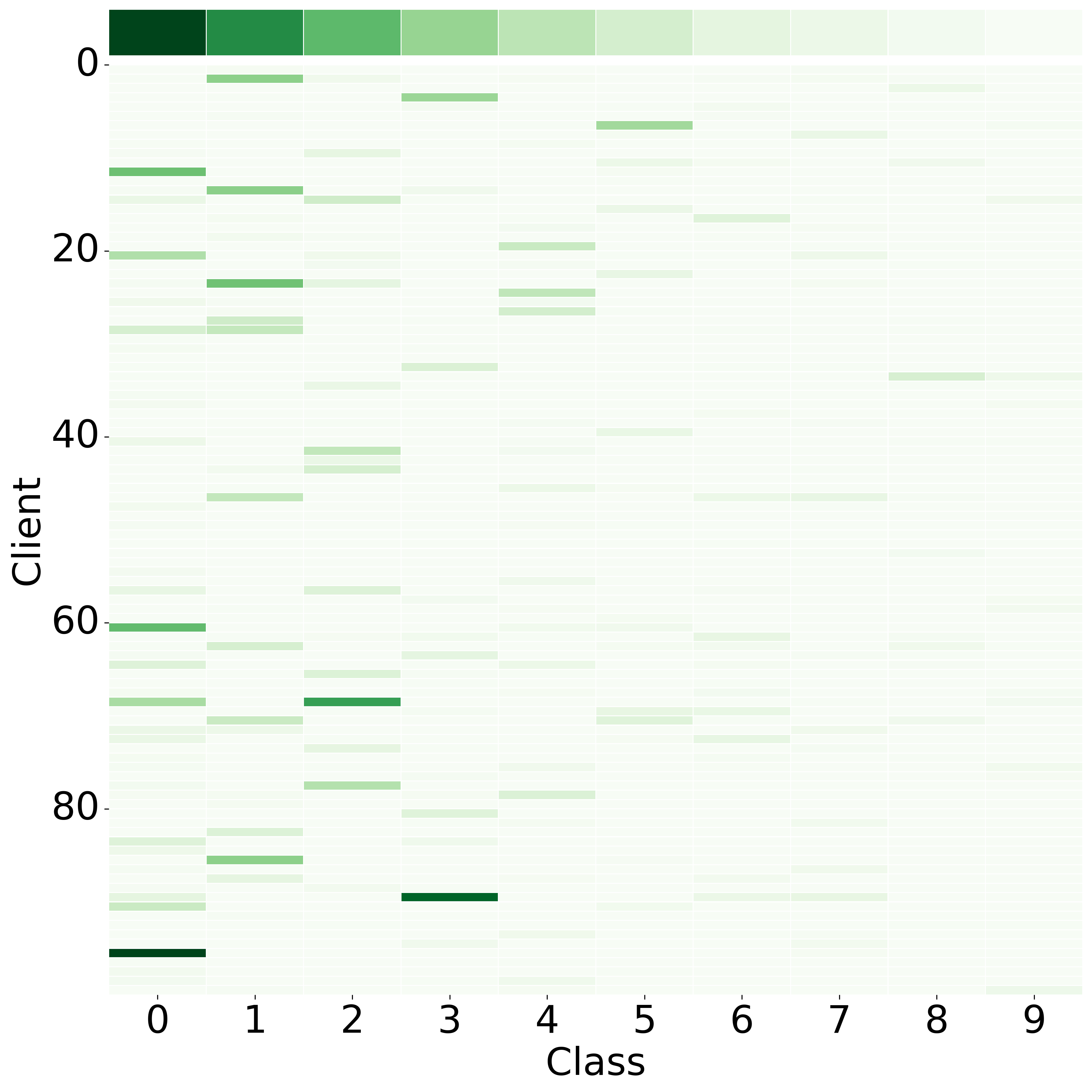}
\caption{Data distribution under the setting of $\beta=0.1$, $IF=0.1$, partitioned according to FedGrab.}
\label{fig:fedgrab_data_distribution}
\end{figure}

The dataset exhibits significant imbalance across clients, with approximately 10\% of clients holding over 50\% of the total samples, while around 40\% of clients possess less than 10\% of the samples. Additionally, certain clients predominantly represent a few classes, leading to skewed class distributions. This imbalance not only necessitates robust methods to ensure equitable participation and effective model aggregation across all clients, but also poses significant challenges for the implementation and weighting of momentum. In scenarios with uneven data quantities, momentum methods may be dominated by a few clients with large datasets, potentially impacting the overall model performance.

\begin{figure}[ht]
\centering
\includegraphics[width=\columnwidth]{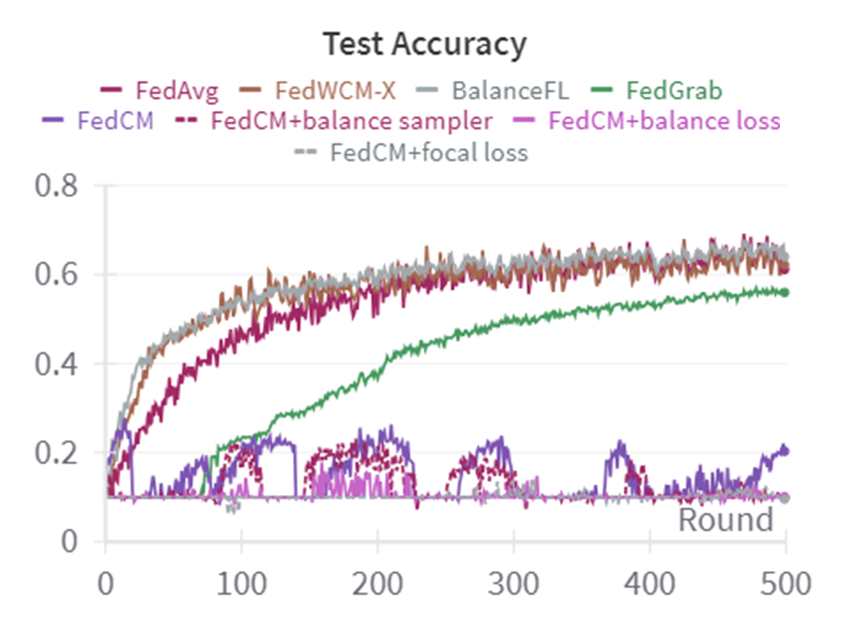}
\caption{Accuracy comparison of our method against six other methods on the dataset.}
\label{fig:accuracy_comparison}
\end{figure}

From Figure 2, we can observe that FedWCM in brown line maintains a high convergence speed in the early stages of training, and its final convergence accuracy is comparable to FedAvg in purple line and BalanceFL in gray line. The lack of its original significant advantage is attributed to the interference caused by weighting based on quantity and inherent weighting, which affects the gradient aggregation effectiveness. On the other hand, the average performance of FedGrab in green line may result from differences in settings such as batch size and number of clients compared to the original experiment. Other potential improvements based on FedCM do not converge.

Additionally, we compare FedAvg, FedCM, and FedWCM-X under this data partitioning scheme with IF settings of 1, 0.4, 0.1, 0.06, 0.04, and 0.01. The comparison is shown in the following table, focusing on the case where $\beta=0.1$.

\begin{table}[ht]
\centering
\vspace{-0.8em}
\caption{Comparison of various approaches under different settings of $IFs$ when $\beta = 0.1$.}
\label{tab:comparison_fedavg_fedcm_1}
\resizebox{\columnwidth}{!}{%
\begin{tabular}{cccccccc}
\toprule
\textbf{} & \textbf{IF} & \textbf{1} & \textbf{0.4} & \textbf{0.1} & \textbf{0.06} & \textbf{0.04} & \textbf{0.01} \\
\midrule
FedAvg & & 0.6802 & 0.7069 & 0.6219 & 0.577 & 0.5502 & 0.4905 \\
FedCM & & 0.6696 & \textbf{0.7405} & \underline{0.2095} & \underline{0.1527} & \underline{0.1438} & \underline{0.1438} \\
FedWCM-X & & \textbf{0.6895} & 0.7346 & \textbf{0.6236} & \textbf{0.5793} & \textbf{0.5632} & \textbf{0.4911} \\
\bottomrule
\end{tabular}%
}
\end{table}


As shown in Table~\ref{tab:comparison_fedavg_fedcm_1}, there are notable performance differences among various approaches under different settings. FedWCM-X consistently outperforms other methods in most scenarios, especially at lower $IF$ values, maintaining the highest accuracy. For instance, at $IF = 0.1$ and $IF = 0.04$, FedWCM-X achieves accuracies of 0.6236 and 0.5632, respectively, significantly surpassing other methods. In contrast, the performance of FedAvg decreases gradually as $IF$ decreases, whereas FedCM performs poorly at low $IF$ values, with a marked drop in accuracy.
%
\section{Exploration of Non-Convergence in Momentum-based Methods}

In this section, we explore the mechanisms behind the non-convergence of momentum-based methods under long-tailed distributions. Attempts to theoretically prove non-convergence have been challenging due to the complexity of deriving inevitable non-convergence conclusions. In centralized algorithms, as discussed in \cite{tang2020long}, causal inference has been used to qualitatively analyze the impact of momentum on imbalanced data, suggesting that "bad" effects can be removed while retaining the "good." This provides some theoretical support, indicating the adverse aspects of momentum on imbalanced data.

%
%

Our analysis is inspired by the Neural Collapse~\cite{fang2021exploring}~\cite{kothapalli2022neural}~\cite{yang2022inducing}, which provides a rigorous mathematical explanation. Neural Collapse describes a scenario in the terminal phase of training deep neural networks, where classifiers and output features form a special geometric structure called the Simplex Equiangular Tight Frame when training samples are balanced. This structure maximizes the angle between features and classifiers of different classes, minimizing inter-class confusion and explaining the excellent generalization and robustness of deep neural networks.

However, when sample numbers are imbalanced, this geometric structure is disrupted, leading to a new phenomenon called Minority Collapse~\cite{fang2021exploring}. Majority classes dominate the loss function, allowing their features and classifiers to span larger angles, while minority classes are compressed, reducing their angles. This phenomenon has been confirmed through mathematical analysis and experiments.


%

Based on this theory, we observed the concentration of neurons across different layers of the neural network during training. We introduce line charts depicting the average concentration changes over rounds for FedAvg, FedCM, and FedWCM under $\beta = 0.1$, $IF = 1$ on the left, and  $\beta = 0.1$, $IF = 0.1$ on the right.

\begin{figure}[h]
\centering
\includegraphics[width=0.48\columnwidth]{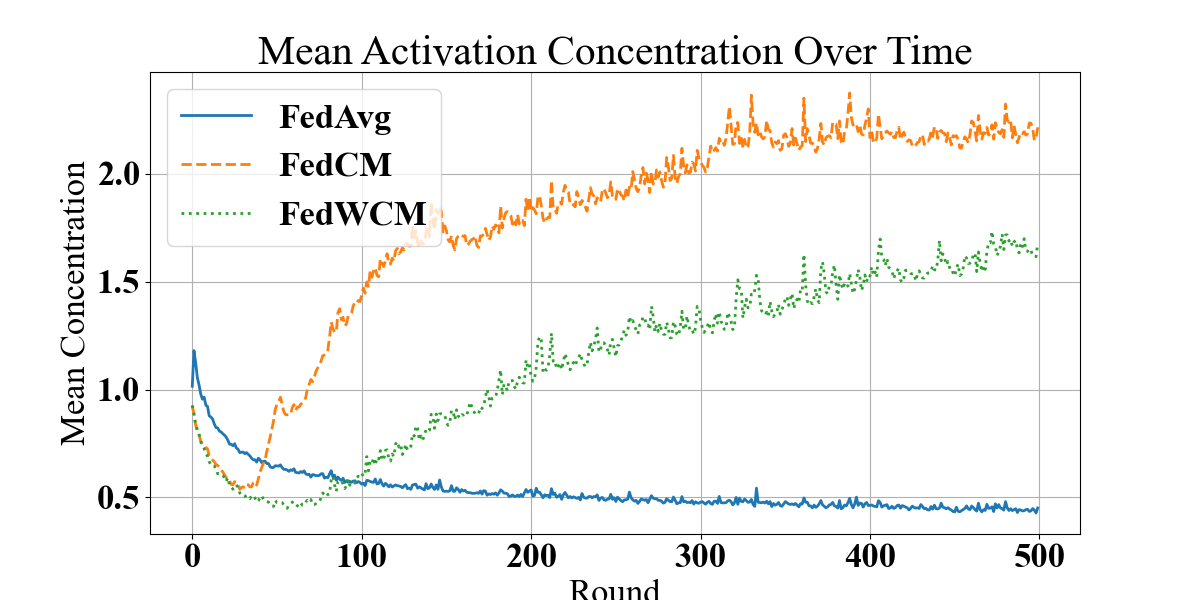} 
\includegraphics[width=0.48\columnwidth]{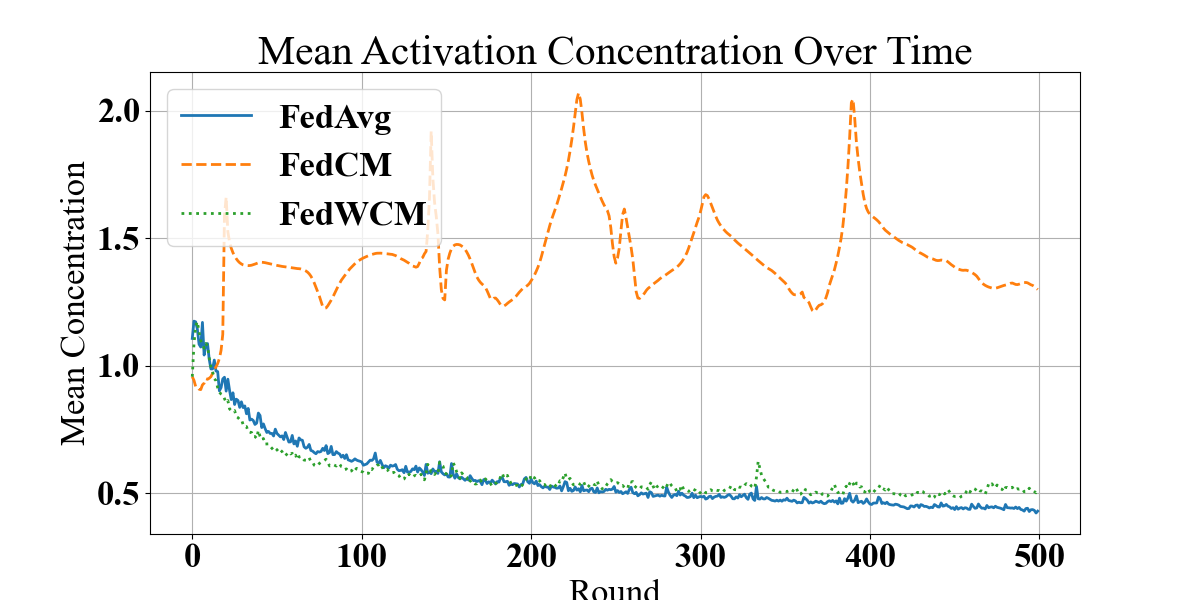} 
\caption{Average neuron concentration over rounds for FedAvg, FedCM, and FedWCM under different settings. Left: $\beta = 0.1$, $IF = 1$. Right: $\beta = 0.1$, $IF = 0.1$.}
\label{fig:average_concentration}
\end{figure}

From the left subfigure with $\beta=0.1$, $IF=0$, it can be observed that the average neuron concentration over rounds for FedAvg in blue line gradually decreases, whereas both FedCM in orange line and FedWCM in green line initially decrease and then increase. This may be due to the accumulation of certain neurons' advantages under the influence of momentum. Furthermore, the increase in FedWCM is relatively smooth. The reason for FedWCM's increase is that we adjust the distillation temperature based on the global imbalance; when the global distribution is fairly balanced, we avoid extreme weighting, as experiments have shown that FedCM performs well in non-long-tailed scenarios. From the right plot with $\beta=0.1$, $IF=0.1$, it is evident that both FedAvg in blue line  and FedWCM in green line exhibit a downward trend in average neuron concentration over rounds, with FedWCM declining faster and more smoothly. In contrast, FedCM in orange line shows periodic large fluctuations.

Next, we introduce three charts showing the detailed concentration changes for each method across layers.

\begin{figure}[ht]
\centering
\includegraphics[width=\columnwidth]{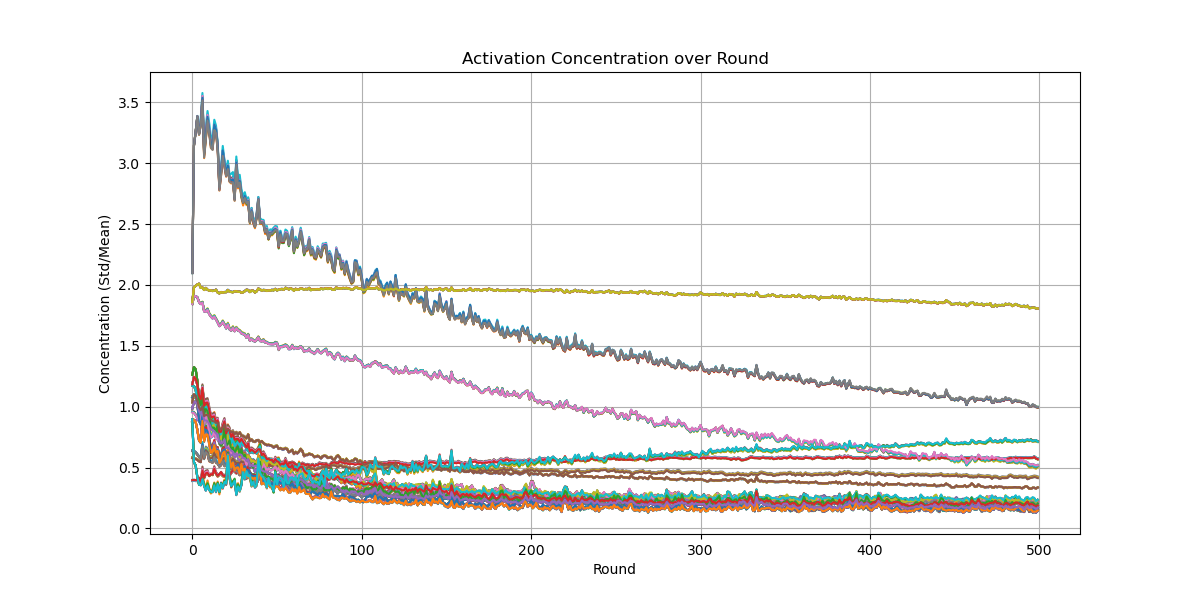} 
\caption{Detailed neuron concentration changes across layers for FedAvg.}
\label{fig:detailed_concentration_fedavg}
\end{figure}

\begin{figure}[ht]
\centering
\includegraphics[width=\columnwidth]{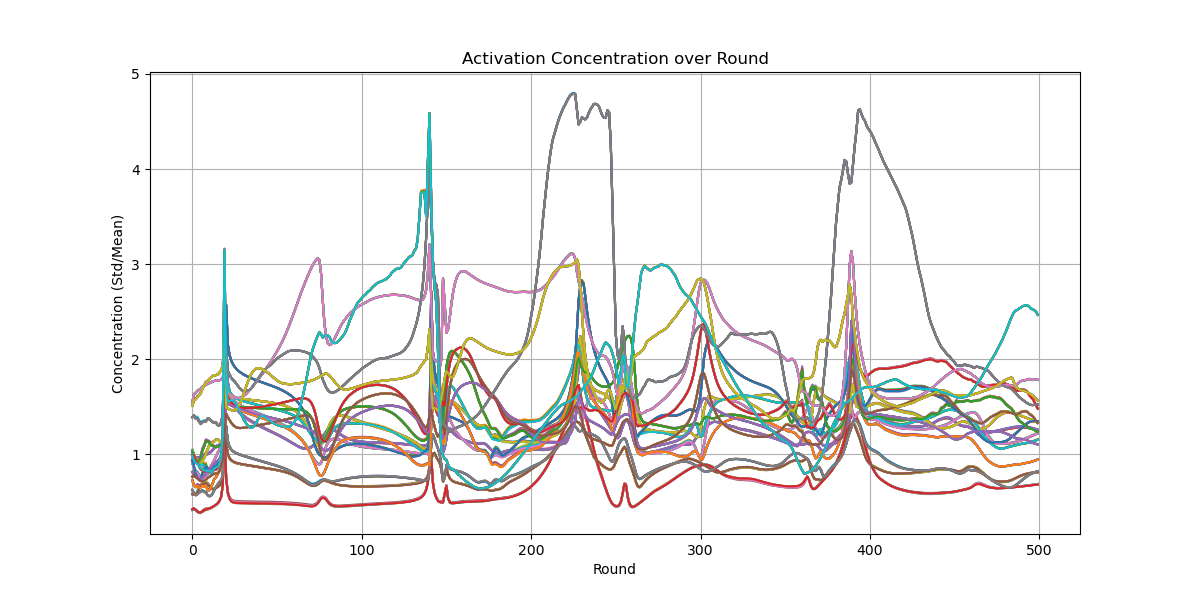} 
\caption{Detailed neuron concentration changes across layers for FedCM.}
\label{fig:detailed_concentration_fedcm}
\end{figure}

\begin{figure}[ht]
\centering
\includegraphics[width=\columnwidth]{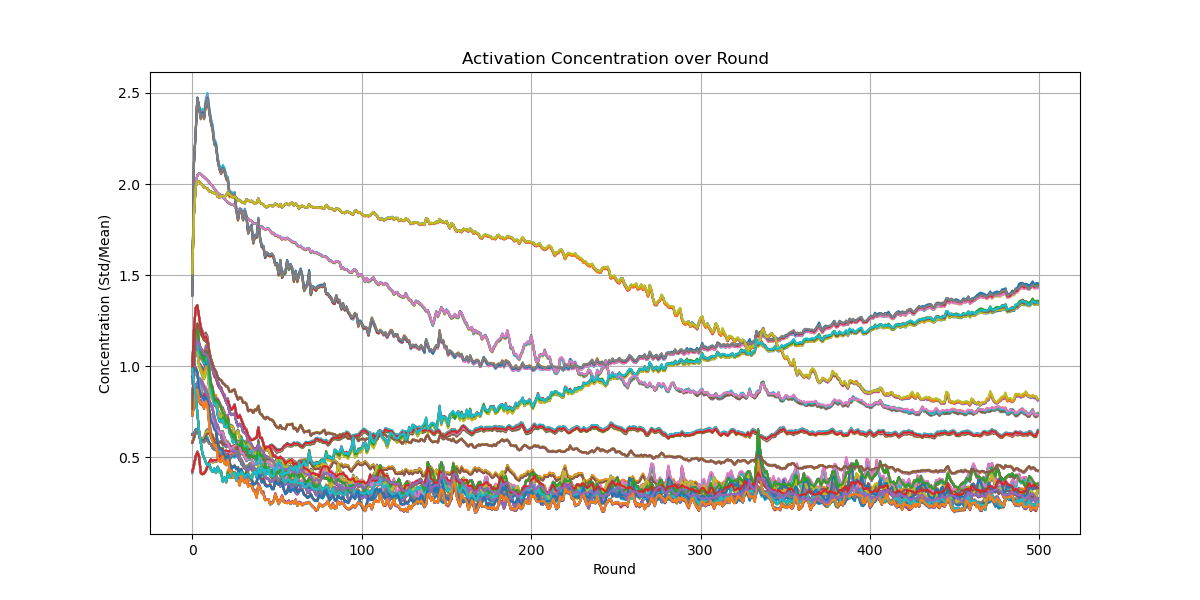} 
\caption{Detailed neuron concentration changes across layers for FedWCM.}
\label{fig:detailed_concentration_fedwcm}
\end{figure}

It can be observed that the neuron concentration in all layers for FedAvg shows a downward trend. In contrast, FedCM exhibits periodic large fluctuations in neuron concentration across all layers, which might be the underlying reason for its difficulty in converging to a stable point. For FedWCM, the neuron concentration mostly decreases across layers, with some layers experiencing an increase, but overall, it remains very stable.

We then focus on analyzing the neuron concentration in FedCM before and after the critical points under long-tailed scenarios. Here, we introduce a combined image showing the accuracy across five long-tailed scenarios and their average concentration change.

\begin{figure}[ht]
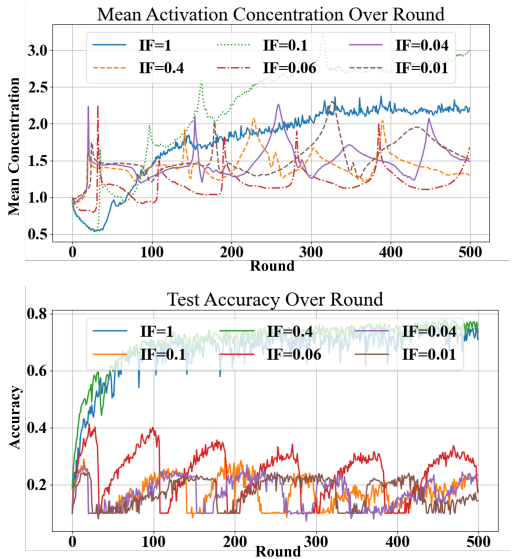

\centering
\includegraphics[width=0.8\columnwidth]{graph/compare_FedCM_IF.png} 
\vspace{0.5cm} 
\includegraphics[width=0.8\columnwidth]{graph/accuracy_plot_IF.png} 
\caption{Top: Average neuron concentration change in FedCM. Bottom: Accuracy across five long-tailed scenarios.}
\label{fig:fedcm_critical_analysis_1}
\end{figure}


By comparing the Average neuron concentration and accuracy change graphs of FedCM under various IF conditions, we observe that as FedCM experiences a precipitous drop in accuracy, its Average neuron concentration also undergoes a sudden change. We believe there is a strong correlation between these two phenomena, possibly due to the occurrence of the minority collapse as discussed in~\cite{fang2021exploring}.

\section{Homomorphic Encryption for Data Distribution in FedWCM}

To protect the privacy of clients’ local class distribution information in FedWCM, we adopt homomorphic encryption (HE), following the protocol used in BatchCrypt~\cite{254465}. HE enables computations directly on encrypted data, ensuring that operations on ciphertexts yield results consistent with computations on plaintexts~\cite{gentry2009fully}.

Specifically, an encryption scheme $E(\cdot)$ is said to be \textit{additively} homomorphic if it satisfies:
\[
E(m_1) \oplus E(m_2) = E(m_1 + m_2),
\]
and \textit{multiplicatively} homomorphic if:
\[
E(m_1) \odot E(m_2) = E(m_1 \cdot m_2),
\]
where $\oplus$ and $\odot$ denote ciphertext-level operations.

In our implementation:

\begin{itemize}
  \item A randomly selected subset of clients generates public/private key pairs and distributes the public keys to other clients.
  \item Each participating client encrypts their local class distribution vector using the public key and uploads the ciphertext to the server.
  \item The server aggregates the ciphertexts and sends the result back to the corresponding key holder for decryption, yielding the global class distribution.
\end{itemize}

This process assumes a \textit{semi-honest} server and does not rely on any trusted third party, aligning with the design goals of BatchCrypt.

Since local class distributions are represented as integer vectors, we use the BFV scheme (Brakerski/Fan-Vercauteren)~\cite{fan2012somewhat}, which supports exact arithmetic over integers. Our implementation is based on the TenSEAL library.

To assess communication overhead, we measured the size of both plaintext and ciphertext representations under varying numbers of classes. As shown in Table~\ref{tab:plaintext_ciphertext}, plaintext size increases linearly with the number of classes, while ciphertext size remains relatively stable at approximately 86KB due to fixed encryption parameters.

\begin{table}[ht]
\centering
\begin{tabular}{|c|c|c|}
\hline
Number of Classes & Plaintext (Byte) & Ciphertext (Byte) \\
\hline
10 & 136 & 88556 \\
20 & 216 & 88554 \\
50 & 456 & 88631 \\
100 & 856 & 88548 \\
\hline
\end{tabular}
\caption{Plaintext and ciphertext sizes for different numbers of classes.}
\label{tab:plaintext_ciphertext}
\end{table}

Notably, since each client only encrypts their own class distribution, the communication cost is independent of the number of clients. For instance, in a scenario with 100 clients and 10-class distributions, the homomorphic encryption process takes approximately 0.0017 seconds per client, with a total communication size of just 13.05MB—negligible compared to model transmission overhead in a typical federated learning round.

In conclusion, our integration of HE into FedWCM enables secure estimation of global class distributions during early training, helping to mitigate class imbalance in long-tailed scenarios while incurring minimal computation and communication overhead.

%
%
\section{Supplementary Experiments}
\subsection{Supplementary Experiments for Heterogeneous methods}

\begin{figure}[ht]
\centering
\includegraphics[width=0.7\columnwidth]{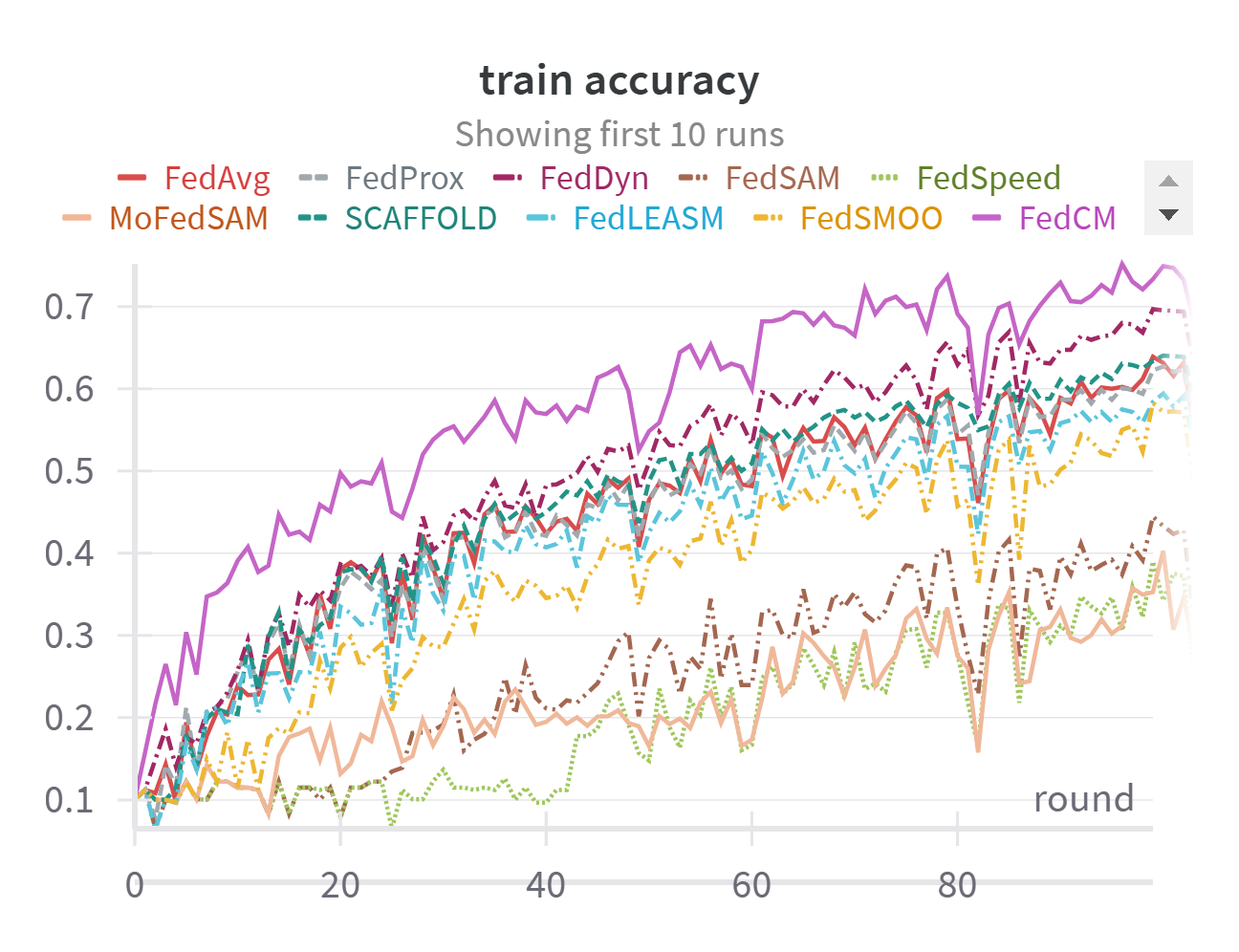} 
\caption{Comparison of Heterogeneous methods for train accuracy.}
\label{fig:comparison_epoch}

\vspace{0.5cm} 

\includegraphics[width=0.7\columnwidth]{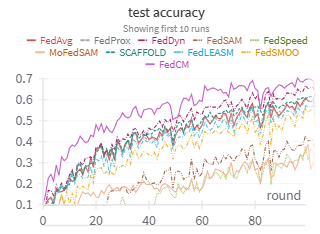} 
\caption{Comparison of Heterogeneous methods for test accuracy.}
\label{fig:comparison_epoch1}
\end{figure}

Figure \ref{fig:comparison_epoch} and Figure \ref{fig:comparison_epoch1} compare the performance of FedCM against nine other federated learning methods under heterogeneous data environments. These methods include FedAvg \cite{mcmahan2017communication}, SCAFFOLD \cite{karimireddy2020scaffold}, FedDyn \cite{acar2021feddyn}, FedProx \cite{li2020federated}, FedSAM \cite{zheng2020fedsam}, MoFedSAM \cite{zheng2020fedsam}, FedSpeed \cite{haddadpour2019fedspeed}, FedSMOO \cite{gupta2021fedsmoo}, and FedLESAM \cite{zheng2020fedsam}. The experiments were conducted on the CIFAR-10 dataset with a heterogeneity level of \( \beta = 0.1 \) (non-long-tailed distribution). As shown in the figures, FedCM not only demonstrates significantly faster convergence but also achieves the highest test accuracy (0.71) at 100 communication rounds, outperforming all other methods. This highlights the strong performance of FedCM in heterogeneous data settings.

Using FedAvg~\cite{mcmahan2017communication} as the baseline, several classical methods, such as FedDyn~\cite{acar2021feddyn}, SCAFFOLD~\cite{karimireddy2020scaffold}, and FedProx~\cite{li2020federated}, achieve higher test accuracy. Specifically, FedDyn, with its dynamic regularization strategy, and SCAFFOLD, which uses control variates to correct local updates, effectively mitigate the local drift caused by data heterogeneity. These methods exhibit relatively smooth convergence curves and achieve slightly higher accuracy than FedAvg. However, their final accuracy still falls short of FedCM.

The three SAM-based methods, including FedSAM~\cite{zheng2020fedsam}, FedSMOO~\cite{sun2023dynamic}, and FedLESAM~\cite{fan2024locally}, focus on improving generalization by flattening the loss landscape. However, as observed in the results, these methods exhibit slower accuracy improvements, particularly during the early stages of training. Similarly, FedSpeed~\cite{haddadpour2019fedspeed} also shows slow progress in the initial stages and overall lags behind other methods, failing to demonstrate its full potential.

In summary, FedCM demonstrates superior performance in heterogeneous data environments, as evidenced by the following two key aspects:
\begin{itemize}
    \item \textbf{Faster Convergence}: FedCM quickly achieves high test accuracy in the early rounds, outperforming other methods in terms of convergence speed.
    \item \textbf{Highest Accuracy}: FedCM achieves a final accuracy of 0.68, surpassing other methods and demonstrating strong stability and generalization capabilities.
\end{itemize}

These findings highlight the effectiveness of FedCM in addressing the challenges posed by data heterogeneity, particularly in non-long-tailed scenarios on the CIFAR-10 dataset. The synergy between momentum mechanisms and consensus updates plays a crucial role in improving convergence speed and achieving superior accuracy.

\subsection{Supplementary Experiments for FedGrab on Cifar10 Dataset}
In this part, We present the experimental results for FedGrab\cite{guo2022fedgrab} on the Cifar10 dataset, which were not included in the main paper due to space limitations. The experiments are conducted using the official reproduction code provided by the authors. These results further demonstrate the effectiveness of FedGrab on image classification tasks.
The supplementary experiments are based on the Cifar10 dataset, a widely-used benchmark for image classification in federated learning. We follow the same experimental settings as described in the main paper, including the number of clients, local epochs, and communication rounds. For completeness, we summarize the key hyperparameters used in Table~\ref{tab:hyperparams}.

\begin{table*}[t]\vspace{-3 mm}
    \centering
    \tiny
    \renewcommand{\arraystretch}{0.85} 
    \setlength{\tabcolsep}{2pt} 
    \caption{Performance comparison on CIFAR-10 under $\beta=0.6$ and $\beta=0.1$ with varying imbalance factors (IF). The reported results represent the mean test accuracy across 3 trials using different random seeds.}
    \begin{tabular}{c c *{16}{>{\centering\arraybackslash}m{0.048\linewidth}}} 
        \toprule
        \multicolumn{2}{c}{} & \multicolumn{2}{c}{FedAvg} & \multicolumn{2}{c}{BalanceFL} & \multicolumn{2}{c}{FedGrab} & \multicolumn{2}{c}{FedCM} & \multicolumn{2}{c}{+ Focal Loss} & \multicolumn{2}{c}{+ Balance Loss} & \multicolumn{2}{c}{+ Balance Sampler} & \multicolumn{2}{c}{FedWCM} \\
        \cmidrule(lr){3-4} \cmidrule(lr){5-6} \cmidrule(lr){7-8} \cmidrule(lr){9-10} \cmidrule(lr){11-12} \cmidrule(lr){13-14} \cmidrule(lr){15-16} \cmidrule(lr){17-18}
        Dataset & IF & 0.6 & 0.1 & 0.6 & 0.1 & 0.6 & 0.1 & 0.6 & 0.1 & 0.6 & 0.1 & 0.6 & 0.1 & 0.6 & 0.1 & 0.6 & 0.1 \\
        \midrule
        \multirow{5}{*}{CIFAR-10} 
        & 1    & 0.7906 & 0.6881  & 0.7629 & 0.6813  & 0.7950 & 0.6813  & 0.8126 & 0.7092  & 0.8040 & 0.6937  & 0.7931 & 0.7169  & 0.8065 & 0.7198  & \textbf{0.8242} & \textbf{0.7337} \\
        & 0.5  & 0.7535 & 0.7183  & 0.7539 & 0.7429  & 0.7810 & 0.6560  & 0.6793 & 0.6686  & 0.6565 & 0.6319  & 0.6877 & 0.6924  & 0.6968 & 0.6590  & \textbf{0.7926} & \textbf{0.7968} \\
        & 0.1  & 0.6232 & 0.6775  & 0.6380 & 0.6541  & 0.6880 & 0.3260  & 0.2175 & 0.2393  & 0.1311 & 0.3095  & 0.1864 & 0.3016  & 0.2871 & 0.3994  & \textbf{0.6905} & \textbf{0.7207} \\
        & 0.05 & 0.5715 & 0.5642  & 0.5652 & 0.5535  & 0.5000 & 0.1870  & 0.2274 & 0.2358  & 0.2005 & 0.1413  & 0.2680 & 0.2525  & 0.1427 & 0.1315  & \textbf{0.6006} & \textbf{0.6132} \\
        & 0.01 & 0.4567 & 0.4600  & 0.4731 & 0.4616  & 0.3140 & 0.1350  & 0.1865 & 0.2312  & 0.1687 & 0.2023  & 0.2087 & 0.2405  & 0.1249 & 0.1584  & \textbf{0.4983} & \textbf{0.5012} \\
        \bottomrule
    \end{tabular}
    \label{tab:hyperparams}
\end{table*}

As shown in Table~\ref{tab:hyperparams}, we present the performance comparison across different methods (FedAvg, BalanceFL, FedGrab, FedCM and its variants, FedWCM) on the CIFAR-10 dataset under different imbalance factors ($IF$) and $\beta$ values of 0.6 and 0.1. Specifically, we focus on the results obtained by FedGrab and compare them with other methods.

From the results, we observe that while FedGrab achieves relatively high accuracy in some cases (e.g., when $IF=1$, $IF=0.1$ and $IF=0.5$), its overall performance is still inferior to FedWCM. FedWCM consistently provides better test accuracy across all imbalance factors, especially in highly imbalanced scenarios ($IF=0.05$ and $IF=0.01$), where it significantly outperforms FedGrab.

While FedGrab performs relatively well under $\beta=0.6$ in some cases, its performance under $\beta=0.1$ suffers significantly. Specifically, in highly heterogeneous data scenarios (e.g., $IF=0.1$ and $IF=0.05$), FedGrab's performance degrades sharply, as shown by the results for $IF=0.1$, where FedGrab achieves only 32.60\% accuracy, compared to 67.75\% for FedAvg and 72.07\% for FedWCM. This indicates that FedGrab is less effective in handling scenarios with increased data heterogeneity.

In contrast, FedWCM demonstrates robustness across all IF values, particularly under $\beta=0.1$, where its results remain consistently superior. This highlights the advantage of FedWCM's weighted aggregation mechanism in mitigating the adverse effects of data heterogeneity.

\section{Proof of Convergence for FedWCM}
\subsection{Notations and Definitions}


Let \( F_0 = \emptyset \) and define \( F_{r,k}^i := \sigma(\{x_{r,j}^i\}_{0 \leq j \leq k} \cup F_r) \) and \( F_{r+1} := \sigma\left( \bigcup_i F_{r,K}^i \right) \) for all \( r \geq 0 \), where \( \sigma(\cdot) \) denotes the \(\sigma\)-algebra. Let \( \mathbb{E}_r[\cdot] := \mathbb{E}[\cdot | F_r] \) represent the expectation conditioned on the filtration \( F_r \), with respect to the random variables \( S_r, \{\xi_{r,k}^i\}_{1 \leq i \leq N, 0 \leq k < K} \) in the \( r \)-th iteration. We also use \( \mathbb{E}[\cdot] \) to denote the global expectation over all randomness in the algorithms.


For all \( r \geq 0 \), we define the following auxiliary variable to facilitate the proofs:
\[
\epsilon_r := \mathbb{E}\left[\|\nabla f(x_r) - g_{r+1}\|^2\right],
\]
where \( g_{r+1} \) represents the aggregated gradient at iteration \( r+1 \).


We further define:
\[
U_r := \frac{1}{NK} \sum_{i=1}^{N} \sum_{k=1}^{K} \mathbb{E}\left[\|x_{r,k}^i - x_r\|^2\right],
\]
where \( x_{r,k}^i \) denotes the \( k \)-th local update of client \( i \) during the \( r \)-th round and \( x_r \) is the global model at round \( r \).


We also introduce:
\[
\zeta_{r,k}^i := \mathbb{E}\left[ x_{r,k+1}^i - x_{r,k}^i \mid F_{r,k}^i \right],
\]
which represents the expected update between successive local updates on client \( i \).


To measure the aggregated local update gradients, we define:
\[
\Xi_r := \frac{1}{N} \sum_{i=1}^{N} \mathbb{E}\left[\|\zeta_{r,0}^i\|^2\right].
\]


Finally, throughout the appendix, let:
\[
\Delta := f(x_0) - f^*, \quad G_0 := \frac{1}{N} \sum_{i=1}^{N} \|\nabla f_i(x_0)\|^2,
\]
where \( f^* \) is the optimal function value, and \( G_0 \) represents the initial gradient norm. Additionally, we set \( x_{-1} := x_0 \) for notational convenience.

\subsection{Preliminary Lemmas}


\begin{lemma}
\label{lemma:sum_of_squares}
Let \( \{X_1, \cdots, X_\tau\} \subset \mathbb{R}^d \) be random variables. If their marginal means and variances satisfy \( \mathbb{E}[X_i] = \mu_i \) and \( \mathbb{E}[\|X_i - \mu_i\|^2] \leq \sigma^2 \), then the following inequality holds:
\[
\mathbb{E}\left[\left\|\sum_{i=1}^\tau X_i\right\|^2\right] \leq \left\|\sum_{i=1}^\tau \mu_i\right\|^2 + \tau^2 \sigma^2.
\]

Additionally, if the random variables are correlated in a Markov way such that \( \mathbb{E}[X_i \mid X_{i-1}, \cdots, X_1] = \mu_i \) and \( \mathbb{E}[\|X_i - \mu_i\|^2] \leq \sigma^2 \), i.e., the variables \( \{X_i - \mu_i\} \) form a martingale, then the following tighter bound applies:
\[
\mathbb{E}\left[\left\|\sum_{i=1}^\tau X_i\right\|^2\right] \leq 2 \mathbb{E}\left[\left\|\sum_{i=1}^\tau \mu_i\right\|^2\right] + 2 \tau \sigma^2.
\]

These results are adapted the work of ~\cite{karimireddy2020scaffold}.
\end{lemma}

\begin{lemma}
Suppose \( \{X_1, \cdots, X_\tau\} \subset \mathbb{R}^d \) be random variables that are potentially dependent. If their marginal means and variances satisfy \( \mathbb{E}[X_i] = \mu_i \) and \( \mathbb{E}[\|X_i - \mu_i\|^2] \leq \sigma^2 \), then it holds that
\[
\mathbb{E}\left[\left\|\sum_{i=1}^\tau X_i\right\|^2\right] \leq \left\|\sum_{i=1}^\tau \mu_i\right\|^2 + \tau^2 \sigma^2.
\]

If they are correlated in the Markov way such that \( \mathbb{E}[X_i \mid X_{i-1}, \cdots, X_1] = \mu_i \) and \( \mathbb{E}[\|X_i - \mu_i\|^2] \leq \sigma^2 \), i.e., the variables \( \{X_i - \mu_i\} \) form a martingale, then the following tighter bound holds:
\[
\mathbb{E}\left[\left\|\sum_{i=1}^\tau X_i\right\|^2\right] \leq 2 \mathbb{E}\left[\left\|\sum_{i=1}^\tau \mu_i\right\|^2\right] + 2 \tau \sigma^2.
\]

These results follow from Lemma 1 in \textit{Scaffold} \cite{karimireddy2020scaffold}.
\end{lemma}

\begin{lemma}
\label{weight}
Let \( x_r \) denote the global model at round \( r \), and let \( x_{r,k}^i \) be the local models for client \( i \) after \( k \) local updates. Assume that the weights \( w_i^r \) are computed using the Softmax function based on the deviation of the local data distribution from the global distribution, i.e.,
\[
w_i^r = \frac{\exp(s_i^r / T)}{\sum_{j=1}^N \exp(s_j^r / T)},
\]
where \( s_i^r \) measures the deviation of client \( i \)'s local distribution from the global distribution. Then, the weighted average of the local gradients,
\[
\frac{1}{K} \sum_{i=1}^N \sum_{k=1}^K w_i^r \nabla f(x_{r,k}^i),
\]
is closer to the global gradient \( \nabla f(x_r) \) than the unweighted average,
\[
\frac{1}{NK} \sum_{i=1}^N \sum_{k=1}^K \nabla f(x_{r,k}^i),
\]
in terms of the \( \ell_2 \)-norm.
\end{lemma}

\begin{proof}
Define the deviations for each client as:
\[
\Delta_i = \nabla f(x_r) - \frac{1}{K} \sum_{k=1}^K \nabla f(x_{r,k}^i),
\]
and let the set of deviations be \( \Delta = (\Delta_1, \Delta_2, \dots, \Delta_N) \). Denote the weights as \( w = (w_1, w_2, \dots, w_N) \), satisfying \( \sum_{i=1}^N w_i = 1 \) and \( w_i \geq 0 \). The unweighted average corresponds to uniform weights \( w_i = \frac{1}{N} \), and the weighted average uses the Softmax weights.

To prove the inequality, we will explicitly use the inverse relationship between \( w_i \) and \( \Delta_i \), along with the properties of the rearrangement inequality.

First, sort \( w_i \) and \( \Delta_i \) such that:
\[
w_1 \leq w_2 \leq \dots \leq w_N \quad \text{and} \quad \Delta_1 \geq \Delta_2 \geq \dots \geq \Delta_N.
\]
Under this ordering, the pairwise product \( w_i \Delta_i \) is minimized compared to any other pairing of \( w_i \) and \( \Delta_i \) due to the rearrangement inequality. Specifically, for any permutation \( \sigma \) of \( \{1, 2, \dots, N\} \), the following holds:
\[
\sum_{i=1}^N w_i \Delta_i \leq \sum_{i=1}^N w_i \Delta_{\sigma(i)}.
\]
Equality is achieved only when \( \Delta_i \) and \( w_i \) are paired in reverse order (i.e., the largest \( \Delta_i \) is matched with the smallest \( w_i \), and so on).

Next, consider the unweighted arithmetic mean \( \frac{1}{N} \sum_{i=1}^N \Delta_i \), which corresponds to the case where all weights are uniform (\( w_i = \frac{1}{N} \) for all \( i \)). For uniform weights, we have:
\[
\frac{1}{N} \sum_{i=1}^N \Delta_i = \sum_{i=1}^N \frac{1}{N} \Delta_i.
\]

Now, compare \( f(\Delta, w) = \sum_{i=1}^N w_i \Delta_i \) to this uniform weighting. By the rearrangement inequality, the weighted sum \( f(\Delta, w) \) is minimized when \( w_i \) and \( \Delta_i \) are inversely related (as given in the problem statement). However, for uniform weights, the weights \( w_i = \frac{1}{N} \) correspond to the mean value of \( \Delta_i \), which is always greater than or equal to the weighted sum \( f(\Delta, w) \) when \( w_i \) and \( \Delta_i \) satisfy the inverse relationship:
\[
\sum_{i=1}^N w_i \Delta_i \leq \sum_{i=1}^N \frac{1}{N} \Delta_i.
\]

Thus, we have:
\[
f(\Delta, w) = \sum_{i=1}^N w_i \Delta_i \leq \frac{1}{N} \sum_{i=1}^N \Delta_i.
\]

Equality holds if and only if all \( \Delta_i \) are equal, in which case the weighting has no effect. This completes the proof.
\qed
\end{proof}
\onecolumn
\subsection{Assumption}
\begin{assumption}
\label{a1}
Each local objective function $f_i$ is $L$-smooth, i.e., for any $x, y \in \mathbb{R}^d$ and $1 \leq i \leq N$, we have
\[
    \|\nabla f_i(x) - \nabla f_i(y)\| \leq L \|x - y\|.
\]
\end{assumption}

\begin{assumption}
\label{a2}
There exists $\sigma \geq 0$ such that for any $x \in \mathbb{R}^d$ and $1 \leq i \leq N$, we have
\[
    \mathbb{E}_{\xi_i} [\nabla F(x; \xi_i)] = \nabla f_i(x),
\]
and
\[
    \mathbb{E}_{\xi_i} [\|\nabla F(x; \xi_i) - \nabla f_i(x)\|^2] \leq \sigma^2,
\]
where $\xi_i \sim \mathcal{D}_i$ are independent and identically distributed.
\end{assumption}

\subsection{Proof of FedWCM}
\begin{lemma}
\label{l2}
If $\gamma L \leq \frac{\alpha_r}{6}$, the following holds for $r \geq 1$:
\[
\epsilon_r \leq \left(1 - \frac{8\alpha_r}{9}\right) \epsilon_{r-1} + \frac{4\gamma^2 L^2}{\alpha_r} \mathbb{E}[\|\nabla f(x_{r-1})\|^2] + \frac{2\alpha_r^2 \sigma^2}{NK} + 4\alpha_r L^2 U_r.
\]
Additionally, it holds for $r=0$ that
\[
\epsilon_0 \leq (1 - \alpha_0)\epsilon_{-1} + \frac{2\alpha_0^2 \sigma^2}{NK} + 4\alpha_0 L^2 U_0.
\]
\end{lemma}

\begin{proof}
For $r \geq 1$, we have
\begin{align*}
\epsilon_r &= \mathbb{E}[\|\nabla f(x_r) - g_{r+1}\|^2] \\
& = \mathbb{E}\left[\|(1 - \alpha_r)(\nabla f(x_r) - g_r) + \alpha_r\left(\nabla f(x_r) - \frac{1}{K}  \sum_i w_i \sum_k \nabla F(x_{r,k}^i; \xi_{r,k}^i)\right)\|^2\right]\\
\end{align*}

Expanding the square, we get
\[
\epsilon_r = \mathbb{E}[\|(1 - \alpha_r)(\nabla f(x_r) - g_r)\|^2] + \alpha_r^2 \mathbb{E}\left[\left\|\nabla f(x_r) - \frac{1}{K}  \sum_{i,k} w_i \nabla F(x_{r,k}^i; \xi_{r,k}^i)\right\|^2\right]
\]
\[
+ 2\alpha_r \mathbb{E}\left[\langle (1 - \alpha_r)(\nabla f(x_r) - g_r), \nabla f(x_r) - \frac{1}{K} \sum_{i,k} w_i \nabla f(x_{r,k}^i) \rangle\right].
\]

Note that $\{\nabla F(x_{r,k}^i; \xi_{r,k}^i)\}_{0 \leq k < K}$ are sequentially correlated. Using the AM-GM inequality, Lemma \ref{lemma:sum_of_squares} and Lemma \ref{weight}, we have
\[
\epsilon_r \leq \left(1 + \frac{\alpha_r}{2}\right) \mathbb{E}[\|(1 - \alpha_r)(\nabla f(x_r) - g_r)\|^2] + 2\alpha_r L^2 U_r + 2\alpha_r^2 \left(\frac{\sigma^2}{NK} + L^2 U_r \right).
\]

Using the AM-GM inequality again and Assumption \ref{a1}, we obtain
\[
\epsilon_r \leq (1 - \alpha_r)^2 \left(1 + \frac{\alpha_r}{2}\right) \epsilon_{r-1} + \left(1 + \frac{\alpha_r}{2}\right) L^2 \mathbb{E}[\|x_r - x_{r-1}\|^2] + 2\alpha_r^2 \frac{\sigma^2}{NK} + 4\alpha_r L^2 U_r.
\]
Substituting $\|x_r - x_{r-1}\|^2 \leq 2\gamma^2 (\|\nabla f(x_{r-1})\|^2 + \|g_r - \nabla f(x_{r-1})\|^2)$ and using $\gamma L \leq \frac{\alpha_r}{6}$, we get
\[
\epsilon_r \leq \left(1 - \frac{8\alpha_r}{9}\right) \epsilon_{r-1} + \frac{4\gamma^2 L^2}{\alpha_r} \mathbb{E}[\|\nabla f(x_{r-1})\|^2] + \frac{2\alpha_r^2 \sigma^2}{NK} + 4\alpha_r L^2 U_r.
\]
Similarly, for $r = 0$,
\[
\epsilon_0 \leq \left(1 + \frac{\alpha_0}{2}\right) \mathbb{E}[\|(1 - \alpha_0)(\nabla f(x_0) - g_0)\|^2] + 2\alpha_0 L^2 U_0 + 2\alpha_0^2 \left(\frac{\sigma^2}{NK} + L^2 U_0\right).
\]
Thus, we have
\[
\epsilon_0 \leq (1 - \alpha_0) \epsilon_{-1} + \frac{2\alpha_0^2 \sigma^2}{NK} + 4\alpha_0 L^2 U_0.
\]
\end{proof}

\begin{lemma}
\label{l3}
If $\eta L K \leq \frac{1}{\alpha_r}$, the following holds for $r \geq 0$:
\[
U_r \leq 2e K^2 \Xi_r + K\eta^2 \alpha_r^2 \sigma^2 \left(1 + 2K^3 L^2 \eta^2 \alpha_r^2\right).
\]
\end{lemma}

\begin{proof}
Recall that $\zeta_{r,k}^i := \mathbb{E}[x_{r,k+1}^i - x_{r,k}^i | F_{r,k}^i] = -\eta ((1 - \alpha_r)g_r + \alpha_r \nabla f_i(x_{r,k}^i))$. Then we have
\[
\mathbb{E}[\|\zeta_{r,j}^i - \zeta_{r,j-1}^i\|^2] \leq \eta^2 L^2 \alpha_r^2 \mathbb{E}[\|x_{r,j}^i - x_{r,j-1}^i\|^2] \leq \eta^2 L^2 \alpha_r^2 (\eta^2 \alpha_r^2 \sigma^2 + \mathbb{E}[\|\zeta_{r,j-1}^i\|^2]).
\]
For any $1 \leq j \leq k-1 \leq K-2$, using $\eta L \leq \frac{1}{\alpha_r K} \leq \frac{1}{\alpha_r(k+1)}$, we have
\[
\mathbb{E}[\|\zeta_{r,j}^i\|^2] \leq \left(1 + \frac{1}{k}\right) \mathbb{E}[\|\zeta_{r,j-1}^i\|^2] + (1 + k) L^2 \eta^4 \alpha_r^4 \sigma^2.
\]
Unrolling the recursive bound and using $\left(1 + \frac{2}{k}\right)^k \leq e^2$, we get
\[
\mathbb{E}[\|\zeta_{r,j}^i\|^2] \leq e^2 \mathbb{E}[\|\zeta_{r,0}^i\|^2] + 4k^2 L^2 \eta^4 \alpha_r^4 \sigma^2.
\]
By Lemma 2, it holds that for $k \geq 2$,
\[
\mathbb{E}[\|x_{r,k}^i - x_r\|^2] \leq 2\mathbb{E}\left[\left(\sum_{j=0}^{k-1} \zeta_{r,j}^i\right)^2\right] + 2k \eta^2 \alpha_r^2 \sigma^2 \leq 2e^2 k^2 \mathbb{E}[\|\zeta_{r,0}^i\|^2] + 2k \eta^2 \alpha_r^2 \sigma^2 \left(1 + 4k^3 L^2 \eta^2 \alpha_r^2\right).
\]
This is also valid for $k = 0, 1$. Summing up over $i$ and $k$ finishes the proof.
\end{proof}

\begin{lemma}
\label{l4}
If $288e(\eta KL)^2 \left((1 - \alpha_r)^2 + e(\alpha_r \gamma LR)^2\right) \leq 1$, then it holds for $r \geq 0$ that
\[
\sum_{r=0}^{R-1} \Xi_r \leq \frac{1}{72eK^2L^2} \sum_{r=-1}^{R-2} (\epsilon_r + \mathbb{E}[\|\nabla f(x_r)\|^2]) + 2\eta^2 \alpha_r^2 eR G_0.
\]
\end{lemma}

\begin{proof}
Note that $\zeta_{r,0}^i = -\eta((1 - \alpha_r)g_r + \alpha_r \nabla f_i(x_r))$, so we have
\[
\frac{1}{N} \sum_{i=1}^N \|\zeta_{r,0}^i\|^2 \leq 2\eta^2 \left((1 - \alpha_r)^2 \|g_r\|^2 + \alpha_r^2 \frac{1}{N} \sum_{i=1}^N \|\nabla f_i(x_r)\|^2\right).
\]
Using Young's inequality, we have for any $q > 0$ that
\[
\mathbb{E}[\|\nabla f_i(x_r)\|^2] \leq (1 + q) \mathbb{E}[\|\nabla f_i(x_{r-1})\|^2] + (1 + q^{-1})L^2 \mathbb{E}[\|x_r - x_{r-1}\|^2].
\]
Summing over $r$ and applying the upper bound of $\eta$ completes the proof.
\end{proof}

\begin{theorem}
Under Assumptions \ref{a1} and \ref{a2}, if we take $g_0 = 0$, $\alpha_r = \min\left(\sqrt{\frac{NKL\Delta}{\sigma^2R}}, 1\right)$ for any constant $c \in (0,1]$, $\gamma = \min\left(\frac{1}{24L}, \frac{\alpha_r}{6L}\right)$, and
\[
\eta KL \lesssim \min \left(1, \frac{1}{\alpha_r \gamma LR}, \left(\frac{L\Delta}{G_0 \alpha_r^3 R}\right)^{1/2}, \frac{1}{\sqrt{\alpha_r N}}, \frac{1}{(\alpha_r^3 NK)^{1/4}}\right),
\]
then FedWCM converges as
\[
\frac{1}{R} \sum_{r=0}^{R-1} \mathbb{E}[\|\nabla f(x_r)\|^2] \lesssim \sqrt{\frac{L\Delta \sigma^2}{NKR}} + \frac{L\Delta}{R}.
\]
Here $G_0 := \frac{1}{N} \sum_{i=1}^N \|\nabla f_i(x_0)\|^2$.
\end{theorem}

\begin{proof}
Combining Lemmas \ref{l2} and \ref{l3}, we have
\[
\epsilon_r \leq \left(1 - \frac{8\alpha_r}{9}\right) \epsilon_{r-1} + 4(\gamma L)^2 \frac{1}{\alpha_r} \mathbb{E}[\|\nabla f(x_{r-1})\|^2] + 2\alpha_r^2 \frac{\sigma^2}{NK} + 4\alpha_r L^2 \left(2e K^2 \Xi_r + K \eta^2 \alpha_r^2 \sigma^2 \left(1 + 2K^3 L^2 \eta^2 \alpha_r^2\right)\right),
\]
and
\[
\epsilon_0 \leq (1 - \alpha_0)E_{-1} + 2\alpha_0^2 \frac{\sigma^2}{NK} + 4\alpha_0 L^2 \left(2e K^2 \Xi_0 + K \eta^2 \alpha_0^2 \sigma^2 \left(1 + 2K^3 L^2 \eta^2 \alpha_0^2\right)\right).
\]
Combining these with Lemma \ref{l4} and applying the choice of $\eta$, $\gamma$, and $\alpha_0$ completes the proof.
\end{proof}
\twocolumn


\end{document}